\documentclass{article} 
\usepackage{tmsty}

\usepackage{xcolor}

\def\eqref#1{Equation~\ref{#1}}

\def\1{\bm{1}}

\DeclareMathAlphabet{\mathsfit}{\encodingdefault}{\sfdefault}{m}{sl}
\SetMathAlphabet{\mathsfit}{bold}{\encodingdefault}{\sfdefault}{bx}{n}

\def\tL{{\tens{L}}}

\newtheorem{theorem}{Theorem}
\newtheorem{lemma}{Lemma}
\newtheorem{corollary}{Corollary}

\newtheorem{assumption}{Assumption}

\newtheorem{definition}{Definition}

\newcommand{\numberthis}{\addtocounter{equation}{1}\tag{\theequation}}

\def\op{\mathop{\mathrm{op}}}
\def\R{\mathbb{R}}

\def\N{\mathbb{N}}

\def\Z{\mathbb{N}}

\def\divergence{\mathrm{div_x}}

\def\ll{L^2(\Omega)}
\def\hh{H^1(\Omega)}
\def\h2{H^2(\Omega)}
\def\linf{L^\infty(\Omega)}

\def\h{H_0^1(\Omega)}

\newcommand{\poly}{\mathrm{poly}}

\def\c4{C_4}

\def\lp{L^p(\Omega)}
\def\hh{H^1(\Omega)}

\def\linf{L^\infty(\Omega)}

\def\tu{\tilde{u}}

\def\tL{\tilde{L}}

\def\barron{\mathcal{B}(\Omega)}

\def\calE{\mathcal{E}}

\def\tcalE{\tilde{\mathcal{E}}}
\def\op{D\calE}
\def\top{D\tcalE}

\def\etavalfullval{\frac{\lambda^4}{4(1 + \pc)^7\Lambda^4}}
\def\oneminusetaval{\frac{\lambda^6}{(1 + \pc)^{10}\Lambda^5}}
\def\etaval{\eta}
\def\pcconst{(1 + \pc)^3}

\def\nbnbu{\partial_{\nabla u}} 
\def\nbnbv{\partial_{\nabla v}} 
\def\nbu{\partial_{u}}
\def\nbv{\partial_{v}}

\def\Dnbnbu{\partial^2_{\nabla u}}
\def\Dnbu{\partial^2_{u}}

\def\domain{\Omega \times \R \times \R^d}
\def\nablax{\nabla_x}

\def\mynabla{\nabla}
\def\Dmynabla{\nabla^2}
\def\Deltax{\Delta_x}
\def\pnbu{\nabla_{(u, \nabla u)}}

\def\barrononestep{\left(1 + \eta (2\pi k_{\tL}d +1)B_{\tL} (2\pi W_t)^{p_{\tL}} \right)\|u\|_{\barron}^{p_{\tL}} + \eta \|f\|_{\barron}}

\def\finalbarronnormtwocol{\left(
\left(1 +\eta 2\pi k_{\tL} W_0 (2\pi k_{\tL}d +1)B_{\tilde{L}}\right)
\left(1 + \eta \|f\|_{\barron}\right)\right)^{pt + \frac{p^t-1}{p-1}}}
\def\finalbarronnormtwocoltwo{\cdot \left(\max\{1, \|u_0\|_{\barron}^{p^t}\}\right)}
\def\errorvalue{
\frac{\epsilon_L R}{\epsilon_L + \Lambda}
\left(\left(1 +\eta(1 + \pc)^2\left(\epsilon_L + \Lambda)\right)\right)^t - 1\right)
}
\def\errorvalueT{
\frac{\epsilon_L R}{\epsilon_L + \Lambda}
\left(\left(1 +\eta(1 + \pc)^2\left(\epsilon_L + \Lambda)\right)\right)^T - 1\right)
}

\def\pc{C_p}

\def\diag{\mathrm{diag}}

\definecolor{prune}{rgb}{0.44, 0.11, 0.11}
\definecolor{myblue}{rgb}{0, .5, 1}
\definecolor{maroon}{rgb}{0.5450, 0, 0}
\definecolor{darkred}{rgb}{0.5450, 0, 0}
\definecolor{RoyalBlue}{RGB}{0,100,170}
\definecolor{DarkBlue}{RGB}{20,70,200}
\definecolor{peach}{rgb}{1, 0.56, 0.56}
\definecolor{NotionGreen}{RGB}{15,123,108}
\definecolor{NotionOrange}{RGB}{217,115,13}
\definecolor{NotionRed}{RGB}{224,62,62}
\definecolor{MontrealBlue}{RGB}{0, 30, 98}

\def\shownotes{1}  
\ifnum\shownotes=1
\newcommand{\authnote}[2]{{$\ll$\textsf{\footnotesize #1: #2}$\gg$}}
\else
\newcommand{\authnote}[2]{}
\fi

\usepackage{natbib}
\usepackage{hyperref}
\usepackage{url}

\title{Neural Network Approximations of PDEs Beyond Linearity: A Representational Perspective}

\author{Tanya Marwah\footnote{\texttt{tmarwah@andrew.cmu.edu}, Carnegie Mellon University. Supported in part by CMU Software Engineering Institute via Department of Defense under contract FA8702-15-D-0002.} \and Zachary C. Lipton\footnote{\texttt{zlipton@andrew.cmu.edu}, Carnegie Mellon University. Supported in part by Amazon AI, Salesforce Research, Facebook, UPMC, Abridge, the PwC Center, the Block Center, the Center for Machine Learning and Health, and the CMU Software Engineering Institute (SEI) via Department of Defense
contract FA8702-15-D-0002.} \and Jianfeng Lu\footnote{\texttt{jianfeng@math.duke.edu}, Duke University. Supported in part by NSF award DMS-2012286.} \and Andrej Risteski\footnote{\texttt{aristesk@andrew.cmu.edu}, Carnegie Mellon University. Supported in part by NSF award IIS-2211907, an Amazon Research Award, and the  CMU/PwC DT\&I Center.}}
\date{}

\begin{document}

\maketitle

\begin{abstract}
A burgeoning line of research leverages
deep neural networks to approximate
the solutions to high dimensional PDEs,
opening lines of theoretical inquiry
focused on explaining how it is 
that these models appear 
to evade the curse of dimensionality.
However, most prior theoretical analyses
have been limited to linear PDEs. 
In this work, we take a step towards studying
the representational power of neural networks 
for approximating solutions to nonlinear PDEs. 
We focus on a class of PDEs known as \emph{nonlinear elliptic variational PDEs}, 
whose solutions minimize an \emph{Euler-Lagrange} energy functional 
$\calE(u) = \int_\Omega L(x, u(x), \nabla u(x)) - f(x) u(x)dx$.
We show that if composing a function with Barron norm $b$ with 
partial derivatives of $L$ produces a function of Barron norm at most $B_L b^p$, 
the solution to the PDE can be $\epsilon$-approximated in the $L^2$ sense by a function with Barron norm $O\left(\left(dB_L\right)^{\max\{p \log(1/ \epsilon), p^{\log(1/\epsilon)}\}}\right)$. By a classical result due to \cite{barron1993universal}, this correspondingly bounds the size of a 2-layer neural network needed to approximate the solution. Treating $p, \epsilon, B_L$ as constants, this quantity is polynomial in dimension, thus showing neural networks can evade the curse of dimensionality. 
Our proof technique
involves neurally simulating (preconditioned) gradient
in an appropriate Hilbert space, which converges exponentially fast to the solution of the PDE, and 
such that we can bound the increase of the Barron norm at each iterate.
Our results subsume and substantially generalize
analogous prior results for linear elliptic PDEs over a unit hypercube. 
\end{abstract}

\section{Introduction}
Scientific applications have become one of the new frontiers 
for the application of deep learning~\citep{jumper2021highly, tunyasuvunakool2021highly, sonderby2020metnet}.
PDEs are a fundamental modeling techniques,
and designing neural networks-aided solvers, particularly in high-dimensions, 
is of widespread usage in many scientific domains \citep{hsieh2019learning, brandstetter2022message}.
One of the most common approaches for applying neural networks
to solve PDEs is to parametrize the solution as a neural network
and minimize a variational objective 
that represents the solution 
\citep{sirignano2018dgm,  yu2017deep}. The hope in doing so is to have a method which computationally avoids the ``curse of dimensionality''---i.e.,
that scales less than exponentially with the ambient dimension.\linepenalty=1000

To date, neither theoretical analysis nor empirical applications
have yielded a precise characterization of the range of PDEs 
for which neural networks-aided methods 
outperform classical methods.
Active research on the empirical side~\citep{han2018solving, weinan2017deep, li2020fourier, li2020neural}
has explored several families of PDEs, 
e.g., Hamilton-Bellman-Jacobi and Black-Scholes, 
where neural networks have been demonstrated
to outperform classical grid-based methods. 
On the theory side, a recent line of works 
\citep{marwah2021parametric, chen2021representation, chen2022regularity} 
has considered the following fundamental question:

\emph{For what families of PDEs, can the solution
be represented by a small neural network?} 

The motivation for this question is computational:
fitting the neural network
(by minimizing some objective) 
is at least as expensive as 
the neural network required to represent it.
Specifically, these works focus on understanding when the approximating neural network can be sub-exponential in size, thus avoiding the curse of dimensionality.
However, to date, these results have only
been applicable to \emph{linear} PDEs. 

In this paper, we take the first step beyond such work,
considering a \emph{nonlinear} family of PDEs
and study \emph{nonlinear variational PDEs}.
These equations have the form 
$-\divergence(\nbnbu L(x, u, \nabla u)) + \nbu L(x, u, \nabla u) = f$
and are a (very general) family of \emph{nonlinear Euler-Lagrange} equations. 
Equivalently, the solution to the PDE is the minimizer of the energy functional 
$\calE(u) = \int_\Omega \left(L(x, u(X), \nabla u(x)) - f(x) u(x) \right)dx$.
This paradigm is very general: it originated with Lagrangian formulations of classical mechanics, and for different $L$, a variety of variational problems can be modeled or learned \citep{schmidt2009distilling, cranmer2020lagrangian}.  
These PDEs
have a variety of  applications in scientific domains,
e.g., (non-Newtonian) fluid dynamics \citep{koleva2018numerical},
meteorology~\citep{weller2016mesh}, 
and nonlinear diffusion equations~\citep{burgers2013nonlinear}.




Our main result is to show that \emph{when 
the function $L$
has ``low complexity'', so does the solution}. 
The notion of complexity we work with 
is the \emph{Barron norm} of the function, 
similar to \citet{chen2021representation,lee2017ability}. This is a frequently used notion of complexity, as 
a function with small Barron norm 
can be represented by a small, 
two-layer neural network, 
due to a classical result \citep{barron1993universal}. 
Mathematically, 
our proof techniques
are based on ``neurally unfolding'' 
an iterative preconditioned gradient descent 
in an appropriate function space:
namely, we show that each of the iterates 
can be represented by a neural network 
with Barron norm not much worse than 
the Barron norm of the previous iterate---along 
with showing a bound on the number of required steps.   

Importantly, our results go beyond the typical non-parametric bounds on the size of an approximator network that can be easily shown by classical regularity results 
of the solution 
to the 
nonlinear variational PDEs
\citep{de1957sulla, nash1957parabolic, nash1958continuity} 
along with universal approximation results
\citep{yarotsky2017error}.
\section{Overview of Results}

Let $\Omega := [0,1]^d$ be a $d$-dimensional hypercube and let $\partial \Omega$ 
denote its boundary. 

We first define the energy functional
whose minimizers are represented by a nonlinear variational PDE---i.e., 
the Euler-Lagrange equation of the energy functional. 

\begin{definition}[Energy functional]
    \label{def:energy_functional}
    For all $u: \Omega \to \R$ such that $u|_{\partial \Omega} = 0$,
    we consider an energy functional of the following form:
    \begin{equation}
        \label{eq:main_energy_functional}
        \calE(u) = \int_\Omega \bigg(L(x, u(x), \nabla u(x)) - f(x) u(x)\bigg) dx,
    \end{equation} 
    where $L: \domain \to \R$
    and there exist  
    constants $0 < \lambda \leq \Lambda$
    such that for every $x \in \Omega$
    the function $L(x, \cdot, \cdot): \R \times \R^d \to \R$
    is smooth and convex, i.e.,
    \begin{equation}
        \label{eq:assumption_hessian_L}
        \diag([0, \lambda \bm{1}_d]) \leq \Dmynabla_{(y,z)} L(x, y, z) \leq \diag([\Lambda, \Lambda \bm{1}_d])
    \end{equation}
    for all $(y, z) \in \R \times \R^d$.
    
    Further, we assume that the function $f: \Omega \to \R$ is such that $\|f\|_{\ll} < \infty$.
    Note that without loss of generality\footnote{
    Since $\lambda$ is a lower bound on the strong convexity constant. 
    If we choose a weaker lower bound, we can always ensure $\lambda \leq 1/\pc$.}
    we assume that 
    $\lambda \leq 1 / \pc$ (where $\pc$ is the Poincare constant defined in Theorem~\ref{thm:poincare_inequality}).
\end{definition}

The minimizer $u^\star$ of the energy functional $\calE$ 
exists and is unique. 
The proof of existence and uniqueness is standard (following essentially along the same lines as Theorem 3.3 in ~\citet{fernandez2020regularity}), and is stated in the following Lemma 
(with the full proof provided in Section~\ref{subsection:proof_of_uniquness} of the Appendix for completeness).

\begin{lemma}
    \label{lemma:proof_of_uniquness}
    Let $L: \domain \to \R$ 
    be the function as defined in Definition~\ref{def:energy_functional}.
    Then the minimizer of the energy functional $\calE$ exists and is unique. 
\end{lemma}

Writing down the condition for stationarity, we can derive a (nonlinear) elliptic PDE for the minimizer of the energy functional in 
Definition~\ref{def:energy_functional}
.
\begin{lemma}
    \label{lemma:derivation_of_nonlinear_variational_PDE}
    Let $u^\star: \Omega \to \R$
    be the unique minimizer for the energy functional 
    in
    Definition~\ref{def:energy_functional}.
    Then for all $\varphi \in \h$,
    $u^\star$ satisfies 
     the following condition:
    \begin{equation}
        \label{eq:weak_solution}
        \begin{split}
        &D\calE[u](\varphi) =
        \int_\Omega \left(\nbnbu L(x, u, \nabla u) \nabla \varphi  + \nbu L(x, u,\nabla u)\varphi - f \varphi \right)dx = 0 , 
        \end{split}
    \end{equation}
    where $d\calE[u](\varphi)$ denotes the directional derivative of the energy functional 
    calculated at $u$ in the direction of $\varphi$.
    Thus, the minimizers of the energy functional 
    satisfy the following PDE with Dirichlet boundary condition: 
    \begin{equation}
        \label{eq:non_linear_PDE_to_solve}
        D\calE(u) := -\divergence (\nbnbu L(x, u, \nabla u)) + \nbu L(x, u, \nabla u)= f 
    \end{equation}
    for all $x \in \Omega$ and $u(x) = 0, \forall x \in \partial \Omega$.
    Here $\divergence$ denotes the divergence operator.
\end{lemma}
The proof for the Lemma can be found in
Appendix~\ref{subsec:proof:derivation_of_nonlinear_variational_PDE}.
Here $-\divergence(\nbnbu L(\nabla \cdot))$ 
and $\nbu L(x, \cdot, \nabla \cdot)$ 
are
operators
that 
acts on a function (in this case $u$).\footnote{
For a vector valued function $F: \R^d \to \R^d$
we will denote the divergence operator either 
by $\divergence F$
or by $\nabla \cdot F$, where
$\divergence F = \nabla \cdot F = \sum_{i=1}^d \frac{\partial_i F}{\partial x_i}$
}

Our goal is to determine if the solution to the PDE in \eqref{eq:non_linear_PDE_to_solve}
can be expressed by a neural network with a small number of parameters.
In order do so, we rely on the concept of a \emph{Barron norm},
which measures the complexity of a function in terms of its Fourier representation.
We show
that if composing with the function partial derivatives of the function $L$ 
increases the Barron norm of $u$ in a bounded fashion, 
then the solution to the PDE in \eqref{eq:non_linear_PDE_to_solve}
will have a bounded Barron norm.
The motivation for using this norm is 
a seminal paper \citep{barron1993universal}, 
which established that any function with Barron norm $C$ 
can be $\epsilon$-approximated by a two-layer neural network 
in the $L^2$ sense by a 2-layer neural network 
with size $O(C^2/\epsilon)$,
thus 
evading the curse of dimensionality
if $C$ is substantially smaller 
than exponential in $d$. 
Informally, we will show the following result: 
\begin{theorem}[Informal] 
Given the function $L$ in Definition~\ref{def:energy_functional}, such that composing a function with Barron norm 
$b$ with $\nbnbu L$ or $\nbu L$
produces a function of Barron norm at most $B_L b^p$ for some constants $B_L, p > 0$. 
Then, 
$\forall \epsilon > 0$,
the minimizer of the energy functional in Definition~\ref{def:energy_functional} can be $\epsilon$-approximated in the $L^2$ sense by a function with Barron norm $$O\left(\left(dB_L\right)^{\max\{p \log(1/ \epsilon), p^{\log(1/\epsilon)}\}}\right).$$ 
\end{theorem}

As a consequence, when $\epsilon, p, B_L$ are thought of as constants, 
we can represent the solution to the Euler-Lagrange PDE    
\eqref{eq:non_linear_PDE_to_solve} by a polynomially-sized network,
as opposed to an exponentially sized network, 
which is what we would get by standard universal approximation results 
and using regularity results for the solutions of the PDE. 

We establish this by neurally simulating 
a preconditioned gradient descent
(for a strongly-convex loss)
in an appropriate Hilbert space,
and show that 
the Barron norm of each iterate---which is a function---is
finite, and at most polynomially
bigger than the Barron norm of the previous iterate. 
We get the final bound by 
(i) bounding the growth of the Barron norm at every iteration;
and (ii) bounding the number of iterations required to reach 
an $\epsilon$-approximation to the solution.
The result in formally stated in Section~\ref{sec:main_result}.
\section{Related Work}
Over the past few years there has been a growing line of work that utilizes neural networks to parameterize the solution to a PDE. 
Works such as \citet{weinan2017deep, yu2017deep, sirignano2018dgm, raissi2017physics}
achieved impressive results on a variety of different applications and
have demonstrated the empirical efficacy of neural networks in solving high dimensional PDEs.
This is a great and promising direction for solving high dimensional PDEs since erstwhile dominant numerical approaches like the 
finite differences and finite element methods \citep{leveque2007finite}
depend primarily upon 
discretizing the input space, hence limiting their use for problems on low dimensional input space.

Several recent works look into the 
theoretical analysis into their representational capabilities
has also gained a lot of attention. 
\citet{khoo2021solving} 
show the existence of a network by discretizing the input space into a mesh and then
using convolutional NNs, where the size of the layers is exponential in
the input dimension.
\citet{sirignano2018dgm} provide
a universal approximation result, showing that
for sufficiently regularized PDEs, there exists a multilayer network that approximates its solution. 
\cite{jentzen2018proof,grohs2020deep,hutzenthaler2020proof} show that 
provided a better-than-exponential
dependence on the input dimension for some specific parabolic PDEs, 
based on a stochastic representation using the Feynman-Kac Lemma, 
thus limiting the applicability of their approach
to PDEs that have such a probabilistic interpretation.

These representational results can 
be further be utilized towards analyzing the generalization properties 
of neural network approximations to PDE solutions.
For example, \citet{lu2021priori} show the generalization analysis for 
the Deep Ritz method for elliptic equations like the Poisson equation and \cite{lu2021prioriSchrodinger}
extends their analysis to the Schrodinger eigenvalue problem.
Furthermore, \citet{mishra2020estimates} look at the generalization properties of 
physics informed neural networks for a linear operators or for non-linear operators
with well-defined linearization. 

Closest to our work is a recent line of study that has focused on families of PDEs for which neural networks evade the curse of dimensionality---i.e. the solution can be approximated by a neural network with a subexponential size. In \citet{marwah2021parametric} the authors show that 
for elliptic PDEs whose coefficients are approximable 
by neural networks with at most $N$ parameters, a neural network exists that $\epsilon$-approximates the solution and has size $O(d^{\log(1/\epsilon)}N)$. 
\citet{chen2021representation} extends this analysis to elliptic PDEs with coefficients with small Barron norm, 
and shows that if the coefficients have Barron norm bounded by $B$, an $\epsilon$-approximate solution exists with Barron norm at most $O(d^{\log(1/\epsilon)}B)$.
The work by \citet{chen2022regularity} derives related results
for the Schr\"odinger 
equation on the whole space. 

As mentioned, while most of
previous
works show key regularity results for neural network 
approximations of solution to PDEs, most of their analysis is limited to simple \emph{linear}
PDEs.
The focus of this paper is towards extending these results to a family of 
PDEs referred to as nonlinear variational PDEs.
This particular family of PDEs consists of many famous PDEs such as $p-$Laplacian (on a bounded domain)
and 
is used to model phenomena like non-Newtonian fluid dynamics and 
nonlinear diffusion processes.
The regularity results for these family of PDEs 
was posed as Hilbert's XIX$^{th}$ problem.
We note that there are classical results like ~\citet{de1957sulla} 
and \citet{nash1957parabolic, nash1958continuity} that provide regularity estimates on the solutions of a
nonlinear variational PDE of the form in \eqref{eq:non_linear_PDE_to_solve}.
One can easily use these regularity estimates, 
along with standard universal approximation results~\cite{yarotsky2017error} 
to show that the solutions can be approximated arbitrarily well.
However, the size of the resulting networks will be exponentially large 
(i.e. they will suffer from the curse of dimensionality)---so are of no use for our desired results.

\section{Notation and Definition}
\label{sec:notations_and_defs}
In this section we introduce some key concepts 
and notation that will be used throughout the paper.
For a vector $x \in \R^d$ we use 
$\|x\|_{2}$ to denote its $\ell_2$ norm.
$C^{\infty}(\Omega)$ is the set of function $f: \Omega \to \R$
that are infinitely differentiable.
For a function $F(x,y,z)$
of multiple variables
we use $\nabla_x F(x,y,z)$ and $\partial_x F(x,y,z)$
to denote the 
(partial) derivative w.r.t the variable $x$
(we drop the subscript if the function takes in only a single variable).
Similarly, $\Deltax$ denotes the Laplacian operator where the derivatives 
are taken w.r.t $x \in \R^d$.
With a slight abuse of notation,
if a function $L:\domain \to \R$
takes functions $u$ and $\nabla u$
as input, we will denote the partial derivatives w.r.t second and third set of coordinates as, 
$\nbu L(x, u, \nabla u)$
and
$\nbnbu L(x, u, \nabla u)$,
respectively.

We also define some important function spaces and associated key results below.
\begin{definition}
For a vector valued function  $g: \R \to \R^{d}$
we define
the $\lp$ norm for $p \in [1, \infty)$ as
$$\|g\|_{\lp} = \left(\int_\Omega \sum_{i}^d \left|g_{i}(x)\right|^pdx\right)^{1/p},$$
For $p = \infty$ we have
$$
\|g\|_{\linf} 
= \max_{i} \|g_{i}\|_{\linf},
$$
\end{definition}

\begin{definition}
    \label{def:lpnorm}
    For a domain $\Omega$,
    the space of functions $\h$ is defined as,
    \begin{align*}
    H^1_0(\Omega):= \{
        & g: \Omega \to \R:  g \in \ll, \\& \;\
            \nabla g \in \ll,  g|_{\partial \Omega} = 0
    \}.
    \end{align*}
    The corresponding norm for $H^1_0(\Omega)$ is defined as,
    $\|g\|_{\h} = \|\nabla g\|_{\ll}.
    $
\end{definition}

Finally, we will make use of the Poincar\'e inequality throughout several of our results.
\begin{theorem}[Poincar\'e inequality, \citet{poincare1890equations}]
    \label{thm:poincare_inequality}
    For any domain $\Theta \subset \R^d$ which is open and bounded,
    there exists a constant $\pc > 0$ such that for all $u \in H_0^1(\Theta)$
    $$\|u\|_{L^2(\Theta)} \leq \pc \|\nabla u\|_{L^2(\Theta)}.$$
\end{theorem}

This constant can be very benignly behaved with dimension for many natural domains---even dimension independent. One such example are convex domains \citep{payne1960optimal}, for which $\pc \leq \pi^2 \mbox{diam}(\Omega)$. Furthermore, for $\Omega = [0,1]^d$, the value of $C_p$ can be explicitly calculated and is equal to 
$1/\pi^2 d$. This is a simple calculation, but we include it for completeness as the following lemma
(proved in Section~\ref{subsec:poincare_proof}):
\begin{lemma}
    \label{lemma:poincare}
    For the domain $\Omega:= [0, 1]^d$, 
    the Poincare constant is equal to 
    $\frac{1}{\pi^2 d}$.
\end{lemma}

\subsection{Barron Norms}
\label{sec:barron_norm}
For a function $f:[0,1]^d \to \R$ 
the Fourier transform is defined as,
\begin{equation}
    \label{eq:fourier_transform}
  \hat{f}(\omega)= 
  \int_{[0,1]^d} f(x)e^{-i 2\pi x^T\omega }dx, \quad \omega \in \Z^d,
\end{equation}
where $\Z^d$ is the set of vectors with natural numbers as coordinates.
The inverse Fourier transform of a function is defined as,
\begin{equation}
    \label{eq:inverse_fourier_transform}
   f(x) = \sum_{\omega \in \Z^d} e^{i 2\pi x^T \omega} \hat{f}(\omega)
\end{equation}
The Barron norm is an average of the norm of the frequency vector weighted by the Fourier magnitude $|\hat{f}(\omega)|$.

\begin{definition}[Spectral Barron Norm, \citep{barron1993universal}]
    \label{def:barron_norm}
    Let $\Gamma$ define a set of functions defined over $\Omega:= [0,1]^d$
    such that $\hat{f}(\omega)$ and $\omega \hat{f}(\omega)$ are absolutely summable, i.e.,
    \begin{align*}
     \Gamma = \Bigg\{f: \Omega \to \R: \sum_{\omega \in \Z^d} |\hat{f}(\omega)|<\infty, \;\&\; 
          \sum_{\omega \in \Z^d} \|\omega\|_2|\hat{f}(\omega)| < \infty
            \Bigg\}
    \end{align*}
    Then we define the spectral Barron norm $\|\cdot\|_{\barron}$ as 
    $$\|f\|_{\barron} =  \sum_{\omega \in \Z^d} (1 +  \|\omega\|_2) |\hat{f}(\omega)|.$$
\end{definition}

The Barron norm can be thought of as an $L_1$ relaxation of requiring sparsity in the Fourier basis---which is intuitively why it confers representational benefits in terms of the size of a neural network required. We refer to \citet{barron1993universal} for a more exhaustive list of the Barron norms of some common function classes.

The main theorem from~\citet{barron1993universal} formalizes this intuition, by bounding the size of a 2-layer network approximating a function with small Barron norm:
\begin{theorem}[Theorem 1, \citet{barron1993universal}]
    \label{theorem:Barrons_theorem}
    Let $f \in \Gamma$ such that $\|f\|_{\barron} \leq C$ 
    and $\mu$ be a probability measure defined over $\Omega$. 
    There exists $a_i \in \R^d$, $b_i \in \R$ and $c_i \in \R$
    such that $\sum_{i=1}^k|c_i| \leq 2C$, 
    there exists a function 
    $f_k(x) = \sum_{i=1}^k c_i \sigma\left(a_i^T x + b_i\right)$, 
    such that we have,
    $$ \int_\Omega \left(f(x) - f_k(x)\right)^2 \mu(dx) \lesssim \frac{C^2}{k}.$$
    Here $\sigma$
    denotes a sigmoidal activation function, i.e., $\lim_{x \to \infty}\sigma(x) = 1$
    and $\lim_{x \to -\infty}\sigma(x) = 0$.
\end{theorem}

Note that while Theorem~\ref{theorem:Barrons_theorem}
is stated for sigmoidal activations like \emph{sigmoid} and \emph{tanh} (after appropriate rescaling), the results are also valid for ReLU activation functions, since $\text{ReLU}(x) - \text{ReLU}(x-1)$
is in fact sigmoidal.
We will also need to work with functions that do not have Fourier coefficients beyond some size (i.e. are band limited), so we introduce the following definition: 
\begin{definition}
    We will define the set $\Gamma_W$ as the set of functions whose 
    Fourier coefficients vanish outside a bounded ball, that is 
    \begin{align*}
    \Gamma_W = &\{f: \Omega \to \mathbb{R}: \mbox{s.t. } f \in \Gamma, \\
    &\;\&\;  
        \forall w, \|w\|_\infty \geq W, \hat{f}(w)= 0 \}.
    \end{align*}
\end{definition}
\smallskip 
Finally, as we will work with vector valued functions, we will also define the Barron norm of a vector-valued function as the maximum of the Barron norms of its coordinates:
\begin{definition} 
\label{definition:barron_norm_vector}
For a vector valued function $g : \Omega \to \R^{d}$, 
we define 
$\|g\|_{\barron} = \max_{i}\|g_{i}\|_{\barron}$.
\end{definition}

\section{Main Result}
\label{sec:main_result}

Before stating the main result we 
introduce the key assumption.

\begin{assumption}
\label{assumption:2}
    The function $L$ in Definition~\ref{def:energy_functional} can be approximated by a function
    $\tL:\domain \to \R$
    such that there exists a constant 
    $\epsilon_L \in [0, \lambda)$ 
    for all $x \in \Omega$ and $u \in \h$
    define $q:=(x, u(x), \nabla u(x)) \in \domain$
    \begin{align*}
    &\qquad \sup_{q} 
        \|\nbu L(q) - \nbu \tL(q)\|_2
        \leq \epsilon_{L} \|u(x)\|_2,\\
        &\text{and,}\;
    \sup_{q} 
        \|\nbnbu L(q) - \nbnbu \tL(q)\|_2
        \leq \epsilon_L \|u(x)\|_2,
    \end{align*}
    Furthermore, we assume that 
    $\tL$ is such that 
    for all $g \in \h$, 
    we have 
    $\tL(x, g, \nabla g) \in \h$, $\tL(x, g, \nabla g) \in \Gamma$ and
    for all $x \in \Omega$
    \begin{equation}
        \label{eq:barron_assumption_barron}
        \begin{split}
            &\|\nbu \tL(x, g, \nabla g)\|_{\barron} \leq B_{\tL} \|g\|_{\barron}^{p_{\tL}}, \\
            \text{and},\;\;
            &\|\nbnbu \tL(x, g, \nabla g)\|_{\barron} \leq B_{\tL} \|g\|_{\barron}^{p_{\tL}}. 
        \end{split}
    \end{equation}
    for some constants $B_{\tL} \geq 0$,
    and $p_{\tL} \geq 0$.
    Finally, if $g \in \Gamma_{W}$ 
    then 
    $\nbu \tL(x, g, \nabla g) \in \Gamma_{k_{\tL}W}$ 
    and 
    $\nbnbu \tL(x, g, \nabla g) \in \Gamma_{k_{\tL}W}$ 
    for a $k_{\tL} >0$.
\end{assumption}
We refer to \emph{Remark 4} for an example 
of how the conditions in the assumption manifest for
a linear elliptic PDE.

This assumption is fairly natural: it states that 
the function $L$ is such that 
its partial derivatives w.r.t $u$ and $\nabla u$
can be approximated (up to $\epsilon_L$)
by a function $\tilde{L}$ 
with 
partial derivatives
that have the property
that when applied to a function $g$ with small Barron norm, the new Barron norm is not much bigger than that of $g$.
The constant $p$ specifies the order of this growth.  
The functions for which our results are most interesting are when the dependence of  $B_{\tL}$ on $d$ is at most polynomial---so that the final size of the approximating network does not exhibit curse of dimensionality. 
For instance, we can take $L$ to be a multivariate polynomial of degree up to $P$:   
we show in Lemma~\ref{lemma:polynomial_barron_norm_result}
the constant
$B_{\tL}$ is $O(d^{P})$ (intuitively, this dependence comes from the total number of monomials of this degree), whereas $p$ and $k$ are both $O(P)$.

With all the assumptions stated, we now 
state our main theorem,
\begin{theorem}[Main Result]
\label{thm:main_theorem}
    Consider the nonlinear variational PDE in \eqref{eq:non_linear_PDE_to_solve}
    which satisfies  Assumption~\ref{assumption:2}
    and let $u^\star \in \h$
    denote the unique solution to the PDE.
    If $u_0 \in \h$
    is a function such that $u_0 \in \Gamma_{W_0}$,
    then for all sufficiently small $\epsilon > 0$, and 
    $$T:= \left\lceil \log \left(\frac{2}{\epsilon}\frac{\calE(u_0) - \calE(u^\star)}{\lambda}
        \right)/
        \log\left(\frac{1}{1 - \oneminusetaval}\right)
        \right\rceil,
    $$
    there exists a function $u_T \in \h$ such that $u_T \in \Gamma_{(2\pi k_{\tL})^TW_0}$
    with Barron norm $\|u_T\|_{\barron}$ bounded by
    \begin{equation}
        \begin{split}
           \finalbarronnormtwocol
           \qquad \finalbarronnormtwocoltwo.
        \end{split}
    \end{equation}
    Furthermore $u_T$ satisfies $\|u_T - u^\star\|_{\h} \leq \epsilon + \tilde{\epsilon}$
    where,
    \begin{align*}
            \tilde{\epsilon} \leq 
            \errorvalueT,
    \end{align*}
    where $R:= \|u^\star\|_{\h} + \frac{1}{\lambda}\calE(u_0)$
    and $\etaval = \etavalfullval$.
\end{theorem}

\emph{Remark 1:} The function $u_0$ can be seen as an initial estimate of the solution, that can be refined to an estimate $u_T$, which is progressively better at the expense of a larger Barron norm. A trivial choice could be $u_0 = 0$, which has Barron norm 1, and which by Lemma~\ref{lemma:properties_of_EandL} would result in $\calE(u_0) \leq \Lambda \|u^*\|^2_{\h}$.  

\emph{Remark 2:} The final approximation error has two terms, 
and note that $T$ goes to infinity as $\epsilon$ tends to zero and
is a consequence of the way $u_T$ is constructed --- by simulating a functional (preconditioned) gradient descent which
converges to the solution 
to the PDE. $\tilde{\epsilon}$ 
stems from the approximation that we make between 
$\tilde{L}$ and $L$, which grows as $T$ increases --- it is a consequence of the fact that the gradient descent updates with $\tilde{L}$ and $L$ progressively drift apart as $T \to \infty$.

\emph{Remark 3:} As in the informal theorem, 
if we think of $p, \Lambda, \lambda, C_p, k, 
\|u_0\|_{\barron}$
as constants, 
the theorem implies that $u^\star$ can be $\epsilon$-approximated in the $L^2$ sense by a function with Barron norm $O\left(\left(dB_L\right)^{\max\{p\log(1/\epsilon), p^{\log(1/\epsilon)}\}}\right)$.
Therefore, combining results from Theorem~\ref{thm:main_theorem} and Theorem~\ref{theorem:Barrons_theorem}
the total number of parameters required to $\epsilon-$approximate the solution $u^\star$
by a $2-$layer neural network is 
$$O\left(\frac{1}{\epsilon^2}
    \left(d B_{L}\right)^{2\max\{p\log(1/\epsilon), p^{\log(1/\epsilon)}\}}\right).
$$

\emph{Remark 4:} The theorem  recovers 
(and vastly generalizes) 
prior results  
which bound 
the Barron norm of
linear elliptic PDEs like \citet{chen2021representation} over the hypercube.
In these results, the elliptic PDE takes 
the form  that for all $u \in \h$,
$-\divergence(A \nabla u) + cu = f$
and the functions $A: \R^d \to \R^{d \times d}$ and 
$c:\R^d \to \R$
are such that $\forall x \in \Omega, A(x)$ is positive definite
and $c(x)$ is non-negative and bounded. 
Further, the functions $A$ and $c$
are assumed to have bounded Barron norm. 
To recover this setting from our result, consider choosing 
\begin{align*}
L(x, u(x), \nabla u(x)) 
:= \frac{1}{2}(\nabla u(x))^TA(x)(\nabla u(x)) + \frac{1}{2}c(x)u(x)^2.
\end{align*}
For this $L$, we have $\Dnbnbu L(x, u(x), \nabla u(x)) = A(x)$
and $\Dnbu L(x, u(x), \nabla u(x)) = c(x)$. 
The conditions in
Equation~\ref{eq:assumption_hessian_L} in Definition~\ref{assumption:2}
require that
$\lambda \leq A(x) \leq \Lambda$
and $0 \leq c(x) \leq \Lambda$, which match the conditions 
on the coefficients $A$ and $c$ in \citet{chen2021representation}.

Further, by a simple application of Lemma~\ref{lemma:barron_norm_algebra}, one can show,
$\|\nbnbu L(x, u, \nabla u)\|_{\barron} \leq d^2\|A\|_{\barron} \|u\|_{\barron},$
and $\|\nbu L(x, u, \nabla u)\|_{\barron} \leq \|A\|_{\barron} \|u\|_{\barron}$
and therefore satisfy \eqref{eq:barron_assumption_barron}
in Assumption~\ref{assumption:2} with $B_{\tilde{L}} = \max\{d^2\|A\|_{\barron}, \|c\|_{\barron}\}$ and $p=1$. Plugging these quantities in Theorem~\ref{thm:main_theorem}, we recover the exact same bound from \citet{chen2021representation}.

\section{Proof of Main Result}
\label{sec:proofs_for_main_result}

The proof will proceed by ``neurally unfolding'' a preconditioned 
gradient descent on the objective $\calE$ in the Hilbert space $\h$. 
This is
inspired by previous works by 
\citet{marwah2021parametric, chen2021representation}
where the authors show that for a linear elliptic PDE, an objective which is quadratic can be designed. In our case, we show that $\calE$ is ``strongly convex'' in some suitable sense --- thus again, bounding the amount of steps needed.

More precisely, the result will proceed in two parts: 
\begin{enumerate}
    \item 
    First, we will show that the sequence of functions $\{u_t\}_{t=0}^\infty$, where
    $u_{t+1} \leftarrow u_t - \eta (I - \Deltax)^{-1} d\calE(u_t)$ can be interpreted as performing preconditioned gradient descent, with the (constant) preconditioner $(I-\Deltax)^{-1}$. We show that in some appropriate sense (Lemma~\ref{lemma:properties_of_EandL}), $\calE$ is strongly convex in $\h$ --- thus the updates converge at a rate of
    $O(\log(1/\epsilon))$.
    \item 
    We then show that the Barron norm of each iterate $u_{t+1}$ 
    can be bounded in terms of the Barron norm of the prior iterate $u_{t}$.
    We show this in Lemma~\ref{lemma:barron_norm_recursion}, where
    we show that given Assumption \ref{assumption:2},
    $\|u_{t+1}\|_{\barron}$ can be bounded as $O(d\|u_t\|^p_{\barron})$.
    By unrolling this recursion we show that 
    the Barron norm of the $\epsilon$-approximation 
    of $u^\star$ is of the order $O(d^{p^T} \|u_0\|^p_{\barron})$ where 
    $T$ are the total steps required for $\epsilon$-approximation and $\|u_0\|_{\barron}$
    is the Barron norm of the first function in the iterative updates.
\end{enumerate}
We now proceed to delineate the main technical ingredients for both of these parts. 
\subsection{Convergence Rate of Sequence}
\label{subsec:convergence_rate_of_sequence}
The proof to show the convergence to the solution $u^\star$
is based on adapting the standard proof (in finite dimension) for convergence of gradient descent when minimizing a strongly convex function $f$. Recall, the basic idea is to Taylor expand $f(x + \delta) \approx f(x) + \nabla f(x)^T \delta  + O(\|\delta\|^2)$. Taking $\delta = \eta \nabla f(x)$, we lower bound the progress term $\eta\|\nabla f(x)\|^2$ using the convexity of $f$, and upper bound the second-order term $\eta^2 \|\nabla f(x)\|^2$ using the smoothness of $f$.  

We follow analogous steps, and prove that we can lower bound the progress term by using some appropriate sense of convexity of $\calE$, and upper bound using some appropriate sense of smoothness of $\calE$, when considered as a function over $\h$. Precisely, we show:

\begin{lemma}[Strong convexity of $\calE$ in $H_0^1$]
    \label{lemma:properties_of_EandL}
    If $\calE, L$ are as in Definition~\ref{def:energy_functional},
    we have
    \begin{enumerate}
   
        \item 
        $ \forall u,v \in \h: \langle \op(u), v\rangle_{\ll}
        = \int_{\Omega} \left(-\divergence(\nbnbu L(x, u, \nabla u)) + \nbu (x, u, \nabla u) \right)v dx 
        = \int_{\Omega} \nbnbu L(x, u, \nabla u) \cdot \nabla v + \nbu L(x, u, \nabla u) v \; dx.
        $     
        \item  
        $\forall u,v \in \h:
            \lambda \|u - v\|^2_{\h} \leq 
            \langle \op(u) - \op(v), u - v\rangle_{\ll}
            \leq (1 + \pc^2)\Lambda \|u - v\|^2_{\h}.
        $
        \item
        $ \forall u,v \in \h:
        \frac{\lambda}{2}\|\nabla v\|_{\ll}^2 + \langle \op(u) - f, v\rangle_{\ll}
            \leq \calE(u + v) - \calE(u) \leq 
        \langle \op(u) - f, v\rangle_{\ll} + 
        \frac{(1 + \pc)^2\Lambda}{2}\|\nabla v\|_{\ll}^2.
        $
        \item $\forall u \in \h:
        \frac{\lambda}{2} \|u - u^\star\|^2_{\h} 
            \leq \calE(u) - \calE(u^\star) 
            \leq \frac{(1 + \pc)^2\Lambda}{2} \|u - u^\star\|_{\h}^2
        $.
    \end{enumerate}
\end{lemma}
Part 1 is a helpful way to rewrite an inner product of a ``direction'' $v$ with $\op(u)$---it is essentially a consequence of integration by parts and the Dirichlet boundary condition. Part 2 and 3 are common proxies of convexity and smoothness: they are ways of formalizing the notion that $\calE$ is strongly convex has ``Lipschitz gradients'', when viewed as a function over $\h$. Finally, Part 4 is a consequence of strong convexity, capturing the fact that if the value of $\calE(u)$ is suboptimal, $u$ must be (quantitatively) far from $u^*$. The proof of the Lemma can be found in Appendix~\ref{sec:appendix:proofs_for_results_for_variational_PDE}.

When analyzing gradient descent in (finite dimensions) to minimize a loss function $\calE$,
the standard condition for progress is that the inner product of the gradient with the direction towards the optimum is lower bounded as $\langle \op(u), u^* - u \rangle_{\ll} \geq \alpha \|u-u^*\|_{\ll}^2$ (we have $\ll$ inner product vs $\h$ norm).
From Parts 2 and 3 of Lemma~\ref{lemma:properties_of_EandL}
one can readily see that the above condition is only satisfied ``with the wrong norm'': i.e. we only have 
$\langle \op(u), u^* - u \rangle_{\ll} \geq \alpha \|u-u^*\|_{\h}^2$. Moreover, since in general, $\|\nabla g\|_{\ll}$ can be arbitrarily bigger than $\|g\|_{\ll}$, there is no way to upper bound the $\h$ norm by the $\ll$ norm. 

We can fix this mismatch by instead doing preconditioned gradient, using the fixed preconditioner $(I-\Deltax)^{-1}$. Towards that, the main lemma about the preconditioner we will need is the following one:  

\begin{lemma}[Norms with preconditioning]
    \label{lemma:I_minus_delta_inverse}
    For all $u \in \h$
    we have
    \begin{enumerate}
        \item 
            $
            \| (I - \Deltax)^{-1} \nablax \cdot \nablax  u\|_{\ll} 
            = \|(I - \Deltax)^{-1} \Deltax u\|_{\ll} 
            \leq \|u\|_{\ll}.
            $
        \item $\|(I - \Deltax)^{-1} u\|_{\ll} \leq \|u\|_{\ll}$
        \item 
            $ \langle (I - \Deltax)^{-1} u, u\rangle_{\ll} \geq \frac{1}{1 + \pc}\langle (-\Deltax)^{-1} u, u\rangle_{\ll}$.
    \end{enumerate}
\end{lemma}
The first part of the lemma is a relatively simple consequence of the fact that $\Deltax$ and $\nablax$ ``commute'', thus can be re-ordered, and the second part that the operator $(I - \Delta_x)^{-1}$ only decreases the $\h$ norm. 
The latter lemma can be understood intuitively as $(I-\Deltax)^{-1}$ and $\Deltax^{-1}$ act as similar operators on eigenfunctions of $\Deltax$ with large eigenvalues (the extra $I$ does not do much) -- and are only different for eigenfunctions for small eigenvalues. However, since the smallest eigenvalue is lower bounded by $1/C_p$, their gap can be bounded. 
 
Combining Lemma~\ref{lemma:properties_of_EandL} and Lemma~\ref{lemma:I_minus_delta_inverse}, we can show that preconditioned gradient descent exponentially converges to 
the solution to the nonlinear variational PDE in~\ref{eq:non_linear_PDE_to_solve}.
\begin{lemma}[Convergence of Preconditioned Gradient Descent]
    \label{lemma:proof_for_convergence}
    Let $u^\star$ denote the unique solution to the PDE in Definition~\ref{eq:non_linear_PDE_to_solve}
    For all $t \in \N$, we define the sequence of functions
    \begin{equation}
        \label{eq:update_equation}
        u_{t+1} \leftarrow u_t - \etaval (I - \Deltax)^{-1}\left(\op(u_t) - f\right).
    \end{equation}
    where $\etaval = \etavalfullval$.
    If $u_0 \in \h$, then 
    after $t$ iterations we have,
    $$ 
        \calE(u_{t+1}) - \calE(u^\star) 
            \leq \left(1 - \oneminusetaval \right)\left(\calE(u_0)-\calE(u^\star)\right).
    $$
\end{lemma}
The complete proof for convergence 
can be found in Section~\ref{subsec:proof_of_proof_of_convergence} of the Appendix.

Therefore, using the result from Lemma~\ref{lemma:properties_of_EandL} part $4$, 
i.e.,
$\|u_t - u^\star\|^2_{\h} \leq \frac{2}{\lambda} \left(\calE(u_t) - \calE(u^\star)\right)$, 
we have 
$$ 
\|u_t - u^\star\|_{\h}^2 
\leq 
\frac{2}{\lambda}\left(1 - \oneminusetaval\right)^t
\left(\calE(u_0) - \calE(u^\star)\right).
$$
and $\|u_T - u^\star\|_{\h}^2 \leq \epsilon$ 
after $T$ steps, where,
\begin{equation}
    T \geq \log \left(\frac{\calE(u_0) - \calE(u^\star)}{\lambda \epsilon/2 }\right)/
        \log\left(\frac{1}{1 - \oneminusetaval}\right).
\end{equation}

\subsection{Bounding the Barron Norm}
\label{subsec:barron_norm_approximation}
Having obtained a sequence of functions that converge
to the solution $u^\star$, we bound the Barron norms of the iterates. We
draw inspiration from \citet{marwah2021parametric, lu2021priori}
and show that
the Barron norm of each iterate in the sequence
increases the Barron norm of the previous iterate in a bounded fashion.
Note that in general, the Fourier spectrum of a composition of functions cannot easily be expressed in terms of the Fourier spectrum of the functions being composed.
However, from Assumption~\ref{assumption:2}
we know 
that the function $L$ can be approximated by $\tL$
such that $\nbnbu \tL(x, u, \nabla u)$ and $ \nbu L(x, u, \nabla u)$
increases the Barron norm of $u$ in a bounded fashion. 
Thus, if we instead of tracking the iterates in \eqref{eq:update_equation} we track
\begin{equation}
    \label{eq:approximate_update_equation}
    \tu_{t+1} = \tu_t - \eta \left(I - \Delta\right)^{-1}\top(\tu_t).
\end{equation}
we can derive the following result (the proof is deferred to Section~\ref{subsec:proof_of_barron_norm_recursion}
of the Appendix):
\begin{lemma}
    \label{lemma:barron_norm_recursion}
    For the updates in \eqref{eq:approximate_update_equation},
    if $\tu_t \in \Gamma_{W_t}$  
    then 
    for all $\eta \in (0, \etaval]$
    we have
    $\tu_{t+1} \in \Gamma_{k_{\tL} W_t}$
    and 
    the Barron norm $\|\tilde{u}_{t+1}\|_{|\barron}$ 
    can be bounded as follows,
    $$\barrononestep.$$
\end{lemma}
The proof consists of 
using the result in \eqref{eq:barron_assumption_barron} 
about the Barron norm of composition of a function with $\tL$, as well as counting the increase
in the Barron norm of a function by any basic algebraic operation, 
as established in Lemma~\ref{lemma:barron_norm_algebra}. Precisely we show:

\begin{lemma}[Barron norm algebra]
    \label{lemma:barron_norm_algebra}
    If $g, g_1, g_2 \in \Gamma$, then the following 
    set of results hold,
    \begin{itemize}
        \item Addition: 
            $\|g_1 + g_2\|_{\barron} \leq \|g_1\|_{\barron} + \|g_2\|_{\barron}$ .
        \item Multiplication: 
            $\|g_1 \cdot g_2\|_{\barron} \leq \|g_1\|_{\barron}\|g_2\|_{\barron}$
        \item Derivative: if $h \in \Gamma_W$ for $i \in [d]$ we have 
            $\|\partial_i g\|_{\barron}\leq 2\pi W\|g\|_{\barron}$.
        \item Preconditioning: if $g \in \Gamma$, then $\|(I - \Delta)^{-1} g\|_{\barron} \leq \|g\|_{\barron}$.
    \end{itemize}
\end{lemma}
The proof for the above lemma can be found in Appendix~\ref{subsec:proof:barron_norm_algebra}.
It bears similarity to an analogous result in \cite{chen2021representation}, 
with the difference being that our bounds are defined in the \emph{spectral} Barron space
which is different from the definition of the Barron norm used in \cite{chen2021representation}. Other than preconditioning, the other properties follow by a straightforward calculation. For preconditioning, the main observation is that $(I-\Delta)^{-1}$ acts as a diagonal operator in the Fourier basis---thus the Fourier coefficients of $(I-\Delta)^{-1}h$ can be easily expressed in terms of those of $h$.

Expanding on the recurrence in Lemma~\ref{lemma:barron_norm_algebra}
we can bound the Barron norm of the function $u_T$ after $T$
iterations as:
\begin{lemma}
    \label{lemma:final_recursion_lemma}
    Given the updates in \eqref{eq:approximate_update_equation}
    and function $u_0 \in \Gamma_{W_0}$ with Barron norm $\|u_0\|_{\barron}$, 
    then after $T$ iterations we have $u_T \in \Gamma_{(2\pi k_{\tL})^TW_0}$
    and $\|u_0\|_{\barron}$
    is bounded by,
    \begin{equation}
    \begin{split}
        \finalbarronnormtwocol
       \qquad \finalbarronnormtwocoltwo
    \end{split}
    \end{equation}
\end{lemma}

Finally, we exhibit a natural class of functions that satisfy the main Barron growth property in Equations~\ref{eq:barron_assumption_barron}. Precisely, we show (multivariate) polynomials of bounded degree have an effective bound on $p$ and $B_L$:
\begin{lemma}
    \label{lemma:polynomial_barron_norm_result}
    Let 
    $
    f(x)
        = \sum_{\alpha, |\alpha|\leq P} \left(A_{\alpha}\prod_{i=1}^d x_i^{\alpha_i}\right)
    $
    where $\alpha$ is a multi-index and 
    $x \in \R^d$.
    If $g: \R^d \to \R^d$
    is such that $g \in \Gamma_W$, 
    then we have $f \circ g \in \Gamma_{PW}$ and 
    the Barron norm can be bounded as
    $
        \|f \circ g\|_{\barron} 
            \leq d^{P/2}\left(\sum_{\alpha, |\alpha|\leq P} |A_{\alpha}|^2 \right)^{1/2}\|g\|_{\barron}^P
    $
\end{lemma}
Hence if $\tL$ 
is a polynomial of degree $P$
then using the fact that for a functions 
$g: \Omega \to \R$
such that 
$g \in \Gamma_{W}$,
from Lemma~\ref{lemma:barron_norm_algebra}
$\max\{\|g\|_{\barron}, \|\nabla g\|_{\barron}\} \leq 2\pi W \|g\|_{\barron}$, 
we will have 
$$\|\tL(x, g, \nabla g)\|_{\barron} 
\leq d^{P/2}\left(\sum_{\alpha, |\alpha|\leq P} |A_{\alpha}|^2 \right)^{1/2}(2\pi W)^P\|g\|_{\barron}^P.$$
Using the \emph{derivative} result from 
Lemma~\ref{lemma:barron_norm_algebra},
the constants in Assumption~\ref{assumption:2}
will take the following values
$B_{\tL} = d^{P/2}(2\pi W)^{P+1}\left(\sum_{\alpha, |\alpha|\leq P} |A_{\alpha}|^2 \right)^{1/2}$, and 
$r=2\pi W P$.


Finally, since we are using an approximation 
of the function $L$
we will incur an error at each step of the iteration. 
The following Lemma shows
that the error between the iterates $u_t$
and the approximate iterates $\tu_t$
increases
with $t$.
The error is calculated by recursively 
tracking the error between 
$u_t$ and $\tu_t$
for each $t$ in terms of the error at $t-1$. 
Note that this error 
can be controlled by using smaller values 
of $\eta$.
\begin{lemma}
    \label{lemma:error_analysis}
    Let $\tL: \R^d \to \R$ be the function satisfying the properties in Assumption~\ref{assumption:2}
    and we have
    \begin{align*}
    &\quad\calE(u) = \int_\Omega L(x, u(x), \nabla u(x)) - f(x)u(x) \; dx \\
    &\text{and}\;\;
    \tcalE(u) = \int_\Omega \tL(x, u(x), \nabla u(x)) - f(x)u(x) dx.
    \end{align*}
    For 
    $\eta \in (0, \etavalfullval]$
    consider the sequences, 
    \begin{align*}
        &\qquad u_{t+1} = u_t - \eta (I - \Delta)^{-1} \op(u_t),\\
        &\text{and,}
        \;\; 
        \tu_{t+1} = \tu_t - \eta (I - \Delta)^{-1} \top(u_t)
    \end{align*}
    then for all $t \in \N$
    and denoting $R:= \|u^\star\|_{\h} + \frac{1}{\lambda}\calE(u_0)$
    we have,
    \begin{align*}
        &\|u_t - \tu_t\|_{\h} \\
        &\leq \errorvalue
    \end{align*}
\end{lemma}

\section{Conclusion and Future Work}
In this work, we take a representational complexity perspective on neural networks, as they are used to approximate solutions of
nonlinear elliptic variational PDEs of the form 
$-\divergence(\nbnbu L(x, u, \nabla u)) + \nbu L(x, u, \nabla u) = f$.
We prove that if $L$ is such that composing partial derivatives of $L$ with function of bounded Barron norm increases the Barron norm in a bounded fashion,  
then 
we can bound the Barron norm of the solution $u^\star$ to the PDE---potentially evading the curse of dimensionality depending on the rate of this increase. 
Our results subsume and vastly generalize prior work on the linear case \citep{marwah2021parametric, chen2021representation}
when the domain is a hypercube. 
Our proof consists of neurally simulating preconditioned gradient descent on the energy function defining the PDE, which we prove is strongly convex in an appropriate sense. 

There are many potential avenues for future work. Our techniques (and prior techniques) strongly rely on the existence of a variational principle characterizing the solution of the PDE. In classical PDE literature, these classes of PDEs are also considered better behaved: e.g. proving regularity bounds is much easier for such PDEs \citep{fernandez2020regularity}. There are many non-linear PDEs that come without a variational formulation for which regularity estimates are derived using non-constructive methods like comparison principles. 
It is a wide open question to construct representational bounds for any interesting family of PDEs of this kind.
It is also a very interesting question to explore other notions of complexity---e.g. number of parameters in a (potentially deep) network like in \cite{marwah2021parametric}, Rademacher complexity, among others. 

\bibliography{example_paper}
\bibliographystyle{plainnat}

\clearpage
\appendix
\section{Proofs from Section~\ref{subsec:convergence_rate_of_sequence}: Convergence Rate of Sequence}
\label{sec:appendix:proofs_for_convergence}



\subsection{Proof of Lemma~\ref{lemma:properties_of_EandL}}
\label{sec:appendix:proofs_for_results_for_variational_PDE}
\begin{proof}

    In order to prove part 1,
    we
    will use the following integration by parts identity,
    for functions $r:\Omega \to \R$ such that 
    and $s: \Omega \to \R$, and $r,s \in \h$,
    \begin{equation}
        \label{eq:change_of_variable_formula}
         \int_\Omega \frac{\partial r}{\partial x_i} s dx 
         = - \int_\Omega r \frac{\partial s}{\partial x_i} dx 
            + \int_{\partial \Omega}rs n d\Gamma
    \end{equation}
    where $n_i$ 
    is a normal at the boundary and 
    $d \Gamma$ is an infinitesimal element
    of the boundary $\partial \Omega$.
        
    Using the formula 
    in \eqref{eq:change_of_variable_formula} 
    for functions $u,v \in \h$, we have
    \begin{align*}
        \langle \op(u), v\rangle_{\ll} &= \left\langle 
        -\nablax \cdot 
        \nbnbu L(x, u, \nabla u) + \nbu L(x, u, \nabla u) , v\right\rangle_{\ll}\\
        &= -\int_{\Omega} \nablax \cdot \nbnbu L(x, u, \nabla u) v  + \nbu L(x, u, \nabla u) v\; dx \\
        &= -\int_{\Omega} \sum_{i=1}^d \frac{\partial \left(\nbnbu L(x, u, \nabla u)\right)_i}{\partial x_i} v  + \nbu L(x, u, \nabla u) v \; dx\\
        &= \int_{\Omega} \sum_{i=1}^d  \left(\nbnbu L(x, u, \nabla u)\right)_i \frac{\partial v}{\partial x_i}  dx 
        + \int_\Omega \sum_{i=1}^d \left(\nbnbu L(x, u, \nabla u)\right)_i v n_i dx + \int_\Omega \nbu L(x, u, \nabla u) v \; dx\\
        &= \int \nbnbu L(\nabla u) \cdot \nabla v  + \nbu L(x, u, \nabla u) v \; dx
    \end{align*}
    where in the last equality
    we use the fact that the function 
    $v \in \h$, thus 
    $v(x) = 0, \forall x \in \partial \Omega$.

   To prove part $2.$ 
   first note from Part $1.$
   we know that 
   $\langle \op(u) - \op(v), u - v\rangle_{\ll}$
   takes the following form,
   \begin{align*}
    &\langle \op(u) - \op(v), u - v\rangle_{\ll} \\
        &= \langle \nbnbu L(x, u, \nabla u) - \nbnbv L(x, v, \nabla v), \nabla u - \nabla v\rangle_{\ll}
        + \langle \nbu L(x, u, \nabla u) - \nbv L(x, v, \nabla v), u - v\rangle_{\ll}
        \numberthis \label{eq:inner_prod_dcalE}
   \end{align*}
    We know that for $x \in \Omega$, we have 
    $$\Dmynabla_{(u,\nabla u)} L(x, u, \nabla u) \leq \diag([\Lambda, \Lambda \bm{1}_d])$$

    Note that 
    $\nabla_{(u, \nabla u)} L(x, u, \nabla u)$
    is a vector, and we can write,
    $\partial_{(u, \nabla u)} L(x, u, \nabla u) = [\nbu L(x, u, \nabla u), \nbnbu L(x, u, \nabla u)]$
    (here for two vectors $a, b$ we define a new vector $c:= [a, b]$ as their concatenation).
    
    Using the smoothness of $L$ can write, 
    \begin{align*}
        &\left[\nbu L(x, u, \nabla u) - \nbu L(x, v, \nabla v), \nbnbu L(x, u, \nabla u)- \nbnbu L(x, v, \nabla v)\right]^T
        \left([u - v, \nabla u - \nabla v]\right) \\
        &\leq [u - v, \nabla u-\nabla v]^T\left(\diag([\Lambda, \Lambda \bm{1}_d])\right)[u -v, \nabla u - \nabla v]\\
        &\leq \Lambda [u - v, \nabla u-\nabla v]^T[u -v, \nabla u - \nabla v]
    \end{align*}
    This implies that for $x \in \Omega$ we have
    \begin{align*}
        &\left(\nbnbu L(x, u(x), \nabla u(x)) - \nbnbu L(x, v(x), \nabla v(x) \right)^T\left(\nabla u(x) - \nabla v(x)\right)
        +
        \left(\nbu L(x, u(x), \nabla u(x)) - \nbu L(x, v(x), \nabla v(x) \right)^T\left(u(x) - v(x)\right)\\
        &\leq \Lambda \|\nabla u(x) - \nabla v(x)\|^2_2  + \Lambda \|u(x) - v(x)\|^2_2
    \end{align*}
    Integrating over $\Omega$ on both sides we get
    \begin{align*}
        &\langle \nbnbu L(x, u, \nabla u) - \nbnbv L(x, v, \nabla v), \nabla u - \nabla v\rangle_{\ll}
        + \langle \nbu L(x, u, \nabla u) - \nbv L(x, v, \nabla v), u - v\rangle_{\ll}\\
        &\leq \Lambda \|\nabla u - \nabla v\|_{\ll}^2 
            + \Lambda \|u - v\|_{\ll}^2\\
        &\leq \Lambda(1 + \pc^2) \cdot \|u - v\|_{\h}^2.
    \end{align*}
    the Poincare inequaltiy from Theorem~\ref{thm:poincare_inequality} in the final equation.
    Hence plugging this result in Equation~\ref{eq:inner_prod_dcalE}
    we have,
    \begin{align*}
    \langle \op(u) - \op(v), u - v\rangle_{\ll} 
    &\leq (\Lambda + \pc^2\Lambda)\|u - v\|_{\h}^2
    \end{align*}
    This proves the right hand side of the inequality in part 2.

   To prove the left and side we use
   similar to the upper bound, using the convexity of the $L(x, \cdot, \cdot) \colon \R \times \R^d$,
   we can lower bound the following term,
    \begin{align*}
        &\left[\nbu L(x, u, \nabla u) - \nbu L(x, v, \nabla v), \nbnbu L(x, u, \nabla u)- \nbnbu L(x, v, \nabla v)\right]^T
        \left([u - v, \nabla u - \nabla v]\right) \\
        &\geq [u - v, \nabla u-\nabla v]^T\left(\diag([0, \lambda \bm{1}_d])\right)[u -v, \nabla u - \nabla v]\\
        &\geq \lambda (\nabla u-\nabla v)^T(\nabla u - \nabla v)\\
    \end{align*}
    Therefore, for all $x \in \Omega$ we have
    \begin{align*}
        &\left(\nbnbu L(x, u(x), \nabla u(x)) - \nbnbu L(x, v(x), \nabla v(x) \right)^T\left(\nabla u(x) - \nabla v(x)\right)\\
        &+
        \left(\nbu L(x, u(x), \nabla u(x)) - \nbu L(x, v(x), \nabla v(x) \right)^T\left(u(x) - v(x)\right)\\
        &\geq \lambda \|\nabla u(x) - \nabla v(x)\|^2_2
    \end{align*}
    Integrating over $\Omega$ on both sides we get
    \begin{align*}
        &\langle \nbnbu L(x, u, \nabla u) - \nbnbv L(x, v, \nabla v), \nabla u - \nabla v\rangle_{\ll}\\
        &+ \langle \nbu L(x, u, \nabla u) - \nbv L(x, v, \nabla v), u - v\rangle_{\ll}\\
        &\geq \lambda \|\nabla u - \nabla v\|_{\ll}^2 \\
        &=\lambda  \|u - v\|_{\h}^2.
    \end{align*}
    Therefore we have,
    \begin{align*}
    \lambda \|u - v\|_{\h}^2 \leq \langle \op(u) - \op(v), u - v\rangle_{\ll} 
    &\leq (\Lambda + \pc^2\Lambda)\|u - v\|_{\h}^2
    \end{align*}
    as we wanted.

    To show part 3, 
    we will again use the fact that the function for a given $x \in \Omega$ the function $L(x, \cdot, \cdot)$
    is strongly convex and smooth.
    Therefore using Taylor's Theorem $L(x, u + v, \nabla u + \nabla v)$ along $L(x, u, \nabla u)$
    we can re-write the energy function as:
    \begin{align*}
        &\calE(u + v) \\
        &= \int_\Omega L(x, u(x) + v(x), \nabla u(x) +  \nabla v(x)) -f(x) (u(x) + v(x))
        dx\\
        &=\int_{\Omega} L(x, u(x), \nabla u(x)) + \pnbu L(x, u(x), \nabla u(x))^T\left[v(x), \nabla v(x)\right] \\
        &\qquad 
            +\frac{1}{2} [v(x), \nabla v(x)]^T \Dmynabla_{(u, \nabla u)}L(\tilde{x}, u(\tilde{x}), \nabla \tilde{x})[u(x), \nabla u(x)]
        -\int f(x)(u(x) + v(x)) dx\\
        &=\int_{\Omega} L(x, u(x), \nabla u(x)) + [\nbu L(u, u(x), \nabla u(x)), \nbnbu L(x, u(x), \nabla u(x))]^T[v(x), \nabla v(x)]\\
        &\qquad 
       + \frac{1}{2} [v(x), \nabla v(x)]^T \Dmynabla_{(u, \nabla u)}L(\tilde{x}, u(\tilde{x}), \nabla u(\tilde{x}))[v(x), \nabla v(x)]
         - \int f(x)(u(x) + v(x)) dx
        \numberthis \label{eq:convexity_proof_1}
    \end{align*}

    From Equation~\ref{eq:assumption_hessian_L} of Definition~\ref{def:energy_functional}
    we know that for a given $x \in \Omega$ the function $L(x, \cdot, \cdot)$ is smooth and convex.
    In particular we know that, 
    $$\diag([0, \lambda I_d]) \leq \nabla^2_{(u, \nabla u)} \leq \diag[\Lambda, \Lambda I_d].$$
    Using this to upper bound \eqref{eq:convexity_proof_1}
    we get,
    \begin{align*}
        \calE(u + v) 
        &\leq\int_{\Omega} L(x, u(x), \nabla u(x)) + [\nbu L(u, u(x), \nabla u(x)), \nbnbu L(x, u(x), \nabla u(x))]^T[v(x), \nabla v(x)]\\
        &\qquad 
       + \frac{\Lambda}{2} [v(x), \nabla v(x)]^T [v(x), \nabla v(x)]
         - \int f(x)(u(x) + v(x)) dx\\
        &=\int_{\Omega} L(x, u(x), \nabla u(x)) + \nbu L(u, u(x), \nabla u(x))v(x) + \nbnbu L(x, u(x), \nabla u(x)) \nabla v(x)\\
        &\qquad 
       + \frac{\Lambda}{2} \left(v(x)^2 + \|\nabla v(x)\|_2^2\right)
         - \int f(x)(u(x) + v(x)) dx\\
        &= \calE(u) + \langle \op(u) -f, v\rangle_{\ll} + \frac{\Lambda}{2}\left(\|v\|_{\ll} + \|v\|_{\h} \right)\\
        \implies \calE(u + v) 
        &\leq \calE(u) + \langle \op(u) -f, v\rangle_{\ll} + \frac{\Lambda (1 + \pc^2)}{2}\|v\|_{\h}
        \numberthis \label{eq:upperbound_energyfunctional}
    \end{align*}
    
   We can similarly lower bound \eqref{eq:convexity_proof_1} by using the convexity of $\Dmynabla_{(u, \nabla u)}L$
   as 
    \begin{align*}
        \calE(u + v) 
        &\geq\int_{\Omega} L(x, u(x), \nabla u(x)) + [\nbu L(u, u(x), \nabla u(x)), \nbnbu L(x, u(x), \nabla u(x))]^T[v(x), \nabla v(x)]\\
        &\qquad 
       + \frac{\Lambda}{2} \nabla v(x)^T \nabla v(x)
         - \int f(x)(u(x) + v(x)) dx\\
        &=\int_{\Omega} L(x, u(x), \nabla u(x)) + \nbu L(u, u(x), \nabla u(x))v(x) + \nbnbu L(x, u(x), \nabla u(x)) \nabla v(x)\\
        &\qquad 
       + \frac{\lambda}{2}  \|\nabla v(x)\|_2^2
         - \int f(x)(u(x) + v(x)) dx\\
        \implies
        \calE(u + v) 
        &\geq \calE(u) + \langle \op(u) -f, v\rangle_{\ll} + \frac{\lambda}{2}\|v\|_{\h}
        \numberthis \label{eq:lowerbound_energyfunctional}
    \end{align*}
    
    Combining \eqref{eq:upperbound_energyfunctional} 
    and \eqref{eq:lowerbound_energyfunctional}
    we get,
    \begin{align*}
        \frac{\lambda}{2}\|\nabla v\|_{\ll}^2 
        + \langle \op(u) - f, v\rangle_{\ll}
            \leq \calE(u + v) - \calE(u) \leq 
        \langle \op(u) - f, v\rangle_{\ll} + 
        \frac{(1 + \pc)^2 \Lambda }{2}\|\nabla v\|_{\ll}^2 
    \end{align*}
    
    Finally, part 4
    follows by plugging in 
    $u = u^\star$ and $v = u - u^\star$ in part 3
    and using the fact that $\op(u^\star) = f$.
\end{proof}

\subsection{Proof of Lemma~\ref{lemma:I_minus_delta_inverse}}
\label{subsec:proof_of_I_minus_delta_invers}

\begin{proof}
    Let $\{\lambda_i,\phi_i\}_{i=1}^\infty$
    denote the (eigenvalue, eigenfunction) pairs of the operator $-\Delta$
    where $0 < \lambda_1 \leq \lambda_2 \leq \cdots$, which are real and countable. ( \citet{evans2010partial}, Theorem 1, Section 6.5)
    
    Using the definition of eigenvalues and eigenfunctions, we have
    \begin{align*}
        \lambda_1 
            &= \inf_{v \in \h} \frac{\langle -\Delta v, v \rangle_{\ll}}{\|v\|_{\ll}^2}\\
            &= \inf_{v \in \h} \frac{\langle \nabla v, \nabla v \rangle_{\ll}}{\|v\|_{\ll}^2}\\
            &= \frac{1}{\pc}.
    \end{align*}
    where in the last equality we use 
    Theorem~\ref{thm:poincare_inequality}.

    Let us write 
    the functions $v, w$ in the eigenbasis as
    $v = \sum_i \mu_i \phi_i$.
    Notice that an eigenfunction of $-\Delta$ is also an eigenfunction for $(I-\Delta)^{-1}$, 
    with correspondinding eigenvalue $\frac{1}{1+\lambda_i}$.
    
    Thus, to show part 1,
    we have,
    \begin{align*}
        \left\|(I -\Delta)^{-1} \nablax \cdot \nabla v\right\|_{\ll}^2
        &=
        \left\|(I - \Delta)^{-1} \Delta v\right\|_{\ll}^2 \\
        &= 
        \left\|\sum_{i=1}^\infty \frac{\lambda_i}{1 + \lambda_i}\mu_i \phi_i\right\|_{\ll}^2\\
        &\leq
        \left\|\sum_{i=1}^\infty \mu_i \phi_i\right\|_{\ll}^2\\
        &= \sum_{i=1}^\infty \mu_i^2 = \|u\|_{\ll}^2
    \end{align*}
    where in the last equality we use the fact that $\phi_i$ are orthogonal.
    
    Now, bounding $\langle (I - \Delta)^{-1} v, v\rangle_{\ll}$ for part $2.$ we use the fact that eigenvalues of the operator 
    $(I - \Delta)^{-1}$
    are of the form $\left\{\frac{1}{1 + \lambda_i}\right\}_{i=1}^\infty$
    we have,
    \begin{align*}
        \langle (I - \Delta)^{-1} v, v\rangle_{\ll} 
            &= \left\langle \sum_{i=1}^\infty \frac{\mu_i}{1 + \lambda_i}\phi_i, 
                \sum_{i=1}^\infty \mu_i \phi_i \right\rangle_{\ll}\\
            &\leq \left\langle \sum_{i=1}^\infty \mu_i\phi_i, 
                \sum_{i=1}^\infty \mu_i \phi_i \right\rangle_{\ll}\\
            &= \|u\|_{\ll}^2
            \numberthis \label{eq:l3_eq2}
    \end{align*}

    Before proving part $3.$,
    note that since $\lambda_1 \leq \lambda_2 \leq \cdots$ and $\frac{x}{1+x}$ is monotonically increasing, 
    we have for all $i\in \N$ 
    \begin{equation}
        \label{eq:l3_eq1}
        \frac{1}{1 + \lambda_i} \geq \frac{1}{(1 + \pc)\lambda_i}
    \end{equation}
    and note that $\frac{1}{\lambda_i}$ are the eigenvalues for $(-\Delta)^{-1}$ for all $i \in \N$.
    Using the inequality in \eqref{eq:l3_eq1} 
    and the fact that $\phi_i'$s are orthogonal, 
    we can further lower bound 
    $\langle (I - \Delta)^{-1} v, v\rangle_{\ll}$
    as follows,
    \begin{align*}
        \langle (I - \Delta)^{-1} v, v\rangle_{\ll} 
            &= \sum_{i=1}^\infty \frac{\mu_i^2}{1 + \lambda_i} \|\phi_i\|_{\ll}^2\\
            &\geq  \sum_{i=1}^\infty \frac{\mu_i^2}{(1 + pc) \lambda_i} \|\phi_i\|_{\ll}^2\\
            &= \frac{1}{1 + \pc}\langle (-\Delta)^{-1} v, v\rangle_{\ll},
    \end{align*}
    where we use the following set of equalities in the last step,
    \begin{equation*} 
        \langle (-\Delta)^{-1} v, v\rangle_{\ll} 
            = \left\langle \sum_{i=1}^\infty \frac{\mu_i}{\lambda_i} \phi_i, 
                \sum_{i=1}^\infty \mu_i\phi_i\right\rangle_{\ll} 
            = \sum_{i=1}^{\infty}\frac{\mu_i^2}{\lambda_i} \|\phi_i\|_{\ll}^2. \qedhere
    \end{equation*}
\end{proof}

\subsection{Proof of Lemma~\ref{lemma:proof_for_convergence}: Convergence of Preconditioned Gradient Descent}
\label{subsec:proof_of_proof_of_convergence}
\begin{proof}
    
    For the analysis we consider 
    $\eta = \etavalfullval$
    
   Taylor expanding as in \eqref{eq:upperbound_energyfunctional}, we have
    \begin{align*}
        \calE(u_{t+1}) 
        &\leq \calE(u_t) - \eta \underbrace{\left\langle  \op(\nabla u_t) - f, 
            (I - \Deltax)^{-1}\left(\op(u_t) - f\right)\right\rangle_{\ll}}_{\text{Term 1}}\\
        & \quad + \underbrace{\frac{\eta^2 \left(1 + \pc\right)^2\Lambda}{2} \left\|\nablax (I - \Deltax)^{-1} \left(\op(u_t) - f\right)\right\|_{\ll}^2}_{\text{Term 2}}.
        \numberthis \label{eq:taylor_expansion_for_ut}
    \end{align*}
    where we have in \eqref{eq:upperbound_energyfunctional} plugged in $u_{t+1} - u_t= -\eta \left(I-\Deltax\right)^{-1} \left(\op(u_t) - f\right)$.
    
    First we lower bound \emph{Term 1.} Since $u^\star$ is the solution to the PDE in \eqref{eq:non_linear_PDE_to_solve},
    we have $\op(u^\star)=f$.
    Therefore we have 
    \begin{equation}
        \label{eq:term1_eq1}
        \left\langle \op(u_t) - f , (I - \Deltax)^{-1} \left(\op(u_t) - f\right)\right\rangle_{\ll}
        =
        \left\langle \op(u_t) - \op(u^\star),  (I - \Deltax)^{-1}
            \left(\op(u_t) - \op(u^\star)\right)\right\rangle_{\ll}
    \end{equation}
    Using the result from Lemma~\ref{lemma:I_minus_delta_inverse}
    part $3.$, we have,
    \begin{align*}
    &\langle \op(u_t) - \op(u^\star), (I - \Deltax)^{-1} \op(u_t) - \op(u^\star) \rangle_{\ll} \\
    &\geq \frac{1}{1 + \pc}\left(\langle \op(u_t) - \op(u^\star), (-\Deltax)^{-1} \op(u_t) - \op(u^\star)\rangle_{\ll}\right)
    \end{align*}
    Using the Equation \eqref{eq:term1_eq1}
    and the fact that 
    $\langle \op(u), v\rangle_{\ll} = \langle \nbnbu L(x, u, \nabla u), \nabla v\rangle_{\ll} + \langle \nbu L(x, u, \nabla u), v\rangle_{\ll}$ from Lemma~\ref{lemma:properties_of_EandL}
    we get,
    \begin{align*}
    &\langle \op(u_t) - \op(u^\star), (I - \Deltax)^{-1} \op(u_t) - \op(u^\star) \rangle_{\ll} \\
    &\geq \frac{1}{1 + \pc}\left(\langle \op(u_t) - \op(u^\star), (-\Deltax)^{-1} \op(u_t) - \op(u^\star)\rangle_{\ll}\right)\\
    &= \frac{1}{1 + \pc}\left(\langle \nbnbu L(x, u_t, \nabla u_t) - \nbnbu L (x, u^\star, \nabla u^\star), 
    \nablax (-\Deltax)^{-1} \left(\op(u_t) - \op(u^\star)\right)\rangle_{\ll}\right)\\
    &\qquad + \frac{1}{1 + \pc}\left(\langle \nbu L(x, u_t, \nabla u_t) - \nbu L(x, u^\star, \nabla u^\star),
            (-\Deltax)^{-1} (\op(u_t) - \op(u^\star)) \rangle_{\ll}\right)\\
    &= \frac{1}{1 + \pc} \left\langle 
        \pnbu L(x, u_t, \nabla u_t) - \pnbu L(x, u^\star, \nabla u^\star), 
        \left[(-\Delta_x)^{-1}\left(\op(u_t) - \op(u^\star)\right), \nabla_x(-\Delta_x)^{-1}\left(\op(u_t) - \op(u^\star)\right)\right]
        \right\rangle_{\ll}
    \numberthis \label{eq:term1_eq_2}
    \end{align*}
    where we combine the terms $\nablax$
    $(-\Delta_x)^{-1}\left(\op(u_t) - \op(u^\star)\right)$ and $\nabla_x(-\Delta_x)^{-1}\left(\op(u_t) - \op(u^\star)\right)$
    into a single vector in the last step.

    Now, note that since for any $x \in \Omega$
    the function $L(x, \cdot, \cdot)$ is strongly convex, we have
    $$\Dmynabla_{(u, \nabla u)} L(x, \nabla u, \nabla x) \geq \diag([0, \lambda \bm{1}_d])$$
    Therefore for all $x$ we can bound 
    $\pnbu L(x, u_t(x), \nabla u_t(x)) - \pnbu L(x, u^\star(x), \nabla u^\star(x))$
    \begin{align*}
        &\pnbu L(x, u_t(x), \nabla u_t(x)) - \pnbu L(x, u^\star(x), \nabla u^\star(x))\\
        &= [u_t(x) - u^\star(x), \nabla u_t(x) - \nabla u^\star(x)]^T \left(\Dmynabla_{(u, \nabla u)} L(\tilde{x}, u(\tilde{x}), \nabla u(\tilde{x})\right) 
        \numberthis \label{eq:lowerbound_inequality_1}
    \end{align*}
    where $\tilde{x} \in \Omega$ (and potentially different from $x$).
    
    Using \eqref{eq:lowerbound_inequality_1} in \eqref{eq:term1_eq_2}, 
    we can lower bound the term as follows:
    \begin{align*}
    &\langle \op(u_t) - \op(u^\star), (I - \Deltax)^{-1} \op(u_t) - \op(u^\star) \rangle_{\ll} \\
    &\geq \frac{1}{1 + \pc} \Bigg\langle 
         [u_t - u^\star, \nabla u_t - \nabla u^\star]^T \left(\Dmynabla_{(u, \nabla u)} L(\tilde{x}, u(\tilde{x}), \nabla u(\tilde{x})\right),\\
        &\qquad 
        \left[(-\Delta_x)^{-1}\left(\op(u_t) - \op(u^\star)\right), \nabla_x(-\Delta_x)^{-1}\left(\op(u_t) - \op(u^\star)\right)
        \Bigg]
        \right\rangle_{\ll}\\
    &\geq \frac{1}{1 + \pc}\left\langle \left[0, \lambda \left(\nabla u_t(x) - \nabla u^\star(x)\right) \right],
        \left[(-\Delta_x)^{-1}\left(\op(u_t) - \op(u^\star)\right), \nabla_x(-\Delta_x)^{-1}\left(\op(u_t) - \op(u^\star)\right)\right]
        \right\rangle_{\ll}\\
    &= \frac{\lambda}{1 + \pc}\left\langle \nabla u_t - \nabla u^\star,
        \nablax (-\Deltax)^{-1}\left(\op(u_t) - \op(u^\star)\right) \right\rangle_{\ll}\\
    &\stackrel{(i)}{=} \frac{\lambda}{1 + \pc}\left\langle (-\Delta) u_t - (-\Delta) u^\star,
        (-\Deltax)^{-1}\left(\op(u_t) - \op(u^\star)\right) \right\rangle_{\ll}\\
    &\stackrel{(ii)}{=} \frac{\lambda}{1 + \pc}\left\langle (-\Delta)^{-1} (-\Delta) u_t - (-\Delta)^{-1}(-\Delta) u^\star,
        \left(\op(u_t) - \op(u^\star)\right) \right\rangle_{\ll}\\
    &\stackrel{(iii)}{=} \frac{\lambda}{1 + \pc}\left\langle  u_t -  u^\star,
        \left(\op(u_t) - \op(u^\star)\right) \right\rangle_{\ll}\\
    &\stackrel{(iv)}{\geq} 
    \frac{\lambda^2}{1 + \pc}\|u_t - u^\star\|_{\h}^2 
    \numberthis \label{eq:lowerbound_term1_eqeq1}
    \end{align*}
    Here, we use the fact 
    that for all $u, v \in \h$ we have
    $\langle \nabla u, \nabla v\rangle_{\ll} = \langle -\Delta u, v\rangle_{\ll}$, i.e., 
    Green's identity (along with the fact that we have a Dirichlet Boundary condition) to get step $(i)$.
    We use the symmetry of the operator $(-\Delta)^{-1}$ in step $(ii)$, and the fact that 
    for a function $g \in \h$
    $(-\Delta)^{-1}(-\Delta) g = g$
    in step $(iii)$.
    We finally use Part $2$ of Lemma~\ref{lemma:properties_of_EandL} in the final step.

    Hence finally Term 1 can be simplified as,
    \begin{align*}
    &\langle \op(u_t) - \op(u^\star), (I - \Deltax)^{-1} \op(u_t) - \op(u^\star) \rangle_{\ll} \\
    &\geq 
    \frac{\lambda^2}{1 + \pc}\|u_t - u^\star\|_{\h}^2 \\
    & \geq \frac{2\lambda^2}{(1 + \pc)^3 \Lambda }\left(\calE(u_t) - \calE(u^\star)\right)
    \end{align*}
    where we use Part $4$ from Lemma~\ref{lemma:properties_of_EandL}
    in the final step.

    We will proceed to upper bounding \emph{Term 2}. 
    Using the definition of $\h$ norm, we can re-write \emph{Term 2} as,
    \begin{align*}
        \left\|\nablax \left(1 - \Deltax\right)^{-1}\left(\op(u_t) - f\right)\right\|^2_{\ll}=
        \left\|\left(1 - \Deltax\right)^{-1}\left(\op(u_t) - f\right)\right\|^2_{\h}
    \end{align*} 
    Writing the $\h$ norm in its variational form (since $\h$ norm is self-adjoint, Lemma~\ref{lemma:self-adjoint})
    and upper bounding it,
    \begin{align*}
        &\left\|\left(1 - \Deltax\right)^{-1}\left(\op(u_t) - f\right)\right\|_{\h} \\
        &= \sup_{\substack{v \in \h\\ \|v\|_{\h} =1}}
        \left\langle\nablax \left(1 - \Deltax\right)^{-1}\left(\op(u_t) - f\right), \nabla v  \right\rangle_{\ll}\\
        &= \sup_{\substack{v \in \h\\ \|v\|_{\h} =1}}
        \left\langle\nablax \left(1 - \Deltax\right)^{-1}\left(\op(u_t) - \op(u^\star)\right), \nabla v  \right\rangle_{\ll}\\
        &\stackrel{(i)}{=} 
        \sup_{\substack{v \in \h\\ \|v\|_{\h} =1}}
        \left\langle \left(1 - \Deltax\right)^{-1}\left(\op(u_t) - \op(u^\star)\right), -\Delta v  \right\rangle_{\ll}\\
        &\stackrel{(ii)}{=}
        \sup_{\substack{v \in \h\\ \|v\|_{\h} =1}}
        \left\langle (-\Delta) \left(1 - \Deltax\right)^{-1}\left(\op(u_t) - \op(u^\star)\right),  v  \right\rangle_{\ll}\\
        &\leq
         \sup_{\substack{v \in \h\\ \|v\|_{\h} =1}}
        \left\langle \op(u_t) - \op(u^\star),  v  \right\rangle_{\ll}
        \numberthis \label{eq:inner_prod_11}
    \end{align*}
    here, step $(i)$ follows from the equality that for all $u, v \in \h$ we have 
    $\langle \nabla u, \nabla v\rangle_{\ll} = \langle -\Delta u, v\rangle_{\ll}$
    and the fact that $-\Delta$ is a symmetric operator in step $(ii)$.

    Finally we use Lemma~\ref{lemma:I_minus_delta_inverse} Part $1$ for the final step. More precisely, 
    we use Part $1$ of Lemma~\ref{lemma:I_minus_delta_inverse} as follows,
    where for a $g \in \h$ we can write, 
    \begin{align*}
     \sup_{\substack{v \in \ll \\ \|v\|_{\ll} = 1}}\langle (-\Delta)(I - \Delta)^{-1} g, v\rangle_{\ll} 
     = \|-\Delta (I - \Delta)^{-1} g\|_{\ll}  \leq \|g\|_{\ll} 
     =: 
     \sup_{\substack{v \in \ll \\ \|v\|_{\ll} = 1}}\langle g, v\rangle_{\ll}
    \end{align*}

    Note that, from Lemma~\ref{lemma:properties_of_EandL}
    we know that for all $u, v$ we can write the inner product $\langle \op(u), v\rangle$ as follows
    \begin{align*}
        \langle \op(u), v\rangle_{\ll} &= 
            \langle \nbnbu L(x, u, \nabla u), v\rangle_{\ll} + 
            \langle \nbu L(x, u, \nabla u), v\rangle_{\ll}\\
        &=
            \langle \pnbu L(x, u, \nabla u), [v, \nabla v]\rangle_{\ll}
    \end{align*}
    that is, we we combine $\nbnbu L $ and $\nbu L$ into a single vector 
    $\pnbu L := [\nbu L(x, u, \nabla u), \nbnbu L(x, u, \nabla u)] \in \R^{d+1}$
    and combining $u$ and $\nabla u$ as a vector $[u, \nabla u]$.

    Using this form and re-writing \eqref{eq:inner_prod_11} and using the fact that for $x \in \Omega$
    $L(x, \cdot, \cdot)$ is convex and smooth in step $(i)$, we have
    \begin{align*}
        &\left\|\left(1 - \Deltax\right)^{-1}\left(\op(u_t) - f\right)\right\|_{\h} \\
        &\leq
         \sup_{\substack{v \in \h\\ \|v\|_{\h} =1}}
        \left\langle \pnbu L(x, u_t, \nabla u_t) 
            - \pnbu L(x, u^\star, \nabla u^\star), [v, \nabla  v]  \right\rangle_{\ll}\\
        &\stackrel{(i)}{=}
         \sup_{\substack{v \in \h\\ \|v\|_{\h} =1}}
        \left\langle \left[u_t - u^\star, \nabla u_t - \nabla u^\star\right]^T \Dmynabla_{(u, \nabla u)} L(\tilde{x}, u(\tilde{x}), \nabla u(\tilde{x})), 
            [ v, \nabla  v]  \right\rangle_{\ll}\\
        &\leq
         \sup_{\substack{v \in \h\\ \|v\|_{\h} =1}}
        \Lambda \left\langle \left[u_t - u^\star, \nabla u_t - \nabla u^\star\right]^T,  
            [ v, \nabla  v]  \right\rangle_{\ll}\\
        &= 
         \sup_{\substack{v \in \h\\ \|v\|_{\h} =1}}
         \Lambda \left\langle u_t - u^\star, v \right\rangle_{\ll}
         + \Lambda \left\langle \nabla (u_t - u^\star), \nabla v \right\rangle_{\ll}\\
        &= 
         \sup_{\substack{v \in \h\\ \|v\|_{\h} =1}}
         \Lambda \pc^2 \|u_t - u^\star\|_{\h}\|v\|_{\h}
         + 
         \Lambda  \|u_t - u^\star\|_{\h}\|v\|_{\h}\\
        &= \Lambda (1 + \pc^2) \|u_t - u^\star\|_{\h}
        \leq \Lambda (1 + \pc)^2 \|u_t - u^\star\|_{\h}
        \numberthis \label{eq:term2_equation_derv_1}
    \end{align*}
    where we use the Poincare Inequality~\ref{thm:poincare_inequality} in the final step.

    Therefore, from the final result in \eqref{eq:term2_equation_derv_1}
    we can upper bound Term $2$ in \eqref{eq:taylor_expansion_for_ut}
    to get,
    \begin{align*}
        \left\|\nablax (I - \Deltax)^{-1} \op(u_t)\right\|_{\ll}^2 
        &\leq 
        \Lambda^2(1 + \pc)^2 \|u_t - u^\star\|_{\h}^2\\
        &\leq \frac{\Lambda^2(1 + \pc)^2}{\lambda}\left(\calE(u_t) - \calE(u^\star)\right)
    \end{align*}
    where we use the result from part $4$ from Lemma~\ref{lemma:properties_of_EandL}.
    
    \begin{align*}
        \implies \calE(u_{t+1}) - \calE(u^\star)  \leq \calE(u_t) - \calE(u^\star) - \left(\frac{2\lambda^2}{\pcconst \Lambda}
        - \eta \frac{(1 + \pc)^4\Lambda^3}{\lambda}\right)\eta \left(\calE(u_t) - \calE(u^\star)\right)
    \end{align*}
    
    Since
    $\eta = \etavalfullval$
    we have
  
    \begin{align*}
        &\calE(u_{t+1}) -\calE(u^\star) \leq \calE(u_t) - \calE(u^\star) - 
        \frac{\lambda^2}{(1 + \pc)^3\Lambda}\eta\left(\calE(u_t) - \calE(u^\star)\right)\\
        \implies&
        \calE(u_{t+1})-\calE(u^\star) 
            \leq \left(1 - \oneminusetaval\right)^t\left(\calE(u_0) - \calE(u^\star)\right). \qedhere
    \end{align*}
\end{proof}
\section{Error Analysis}
\label{section:error Analysis}

\subsection{Proof of Lemma~\ref{lemma:error_analysis}}

\begin{proof}
    We define for all $t$
    $r_t = \tu_t - u_t$, and will iteratively bound $\|r_t\|_{\ll}$.
    
    Starting with $u_0=0$ and $\tu_t=0$, we define the iterative sequences as,
    $$ 
    \begin{cases}
        u_0 = u_0 \\
        u_{t+1} = u_t - \eta(I - \Deltax)^{-1} D\calE(u_t)
    \end{cases}
    $$
    $$ 
    \begin{cases}
        \tu_t = u_0 \\
        \tu_{t+1} = \tu_t - \eta(I - \Deltax)^{-1}D\tcalE(\tu_t)
    \end{cases}
    $$
    where 
    $\eta \in \left(0, \etavalfullval\right]$.
    Subtracting the two we get,
    \begin{align*}
        \tu_{t+1} - u_{t+1} = \tu_t - u_t - \eta (I - \Deltax)^{-1}\left(D\tcalE(\tu_t) - D\calE(u_t)\right)\\
        \implies 
        r_{t+1} = r_t - \eta (I - \Deltax)^{-1}\left(D\tcalE(u_t + r_t) -  D\calE(u_t)\right) 
        \numberthis \label{eq:residual_equation}
    \end{align*}
 
    Taking $\h$ norm on both sides we get,
    \begin{equation}
        \label{eq:update_equation}
        \|r_{t+1}\|_{\h} \leq \|r_t\|_{\h} 
            + \eta \left\|(I - \Deltax)^{-1} \left(D\tcalE(u_t + r_t) - D\calE(u_t)\right)\right\|_{\h}
    \end{equation}
    
    Towards bounding 
    $\left\|(I - \Deltax)^{-1}D\tcalE(u_t + r_t) - D\calE(u_t)\right\|_{\h}$,
    from Lemma~\ref{lemma:dual_norm} 
    we know that the dual norm of $\|w\|_{\h}$ is $\|w\|_{\h}$,
    thus,
    \begin{align*}
        &\left\|(I - \Deltax)^{-1} D\tcalE(u_t + r_t) - D\calE(u_t)\right\|_{\h}\\
        &=  \sup_{\substack{\varphi \in \h \\ \|\varphi\|_{\h} = 1}}
            \left\langle \nabla (I - \Deltax)^{-1} \left(D\tcalE(u_t + r_t) 
                - D\calE(u_t)\right), \nabla \varphi  \right\rangle_{\ll} \\
        &=  \sup_{\substack{\varphi \in \h \\ \|\varphi\|_{\h} = 1}}
            \left\langle \nabla (I - \Deltax)^{-1} \left(D\tcalE(u_t + r_t) 
                - D\calE(u_t + r_t)\right), \nabla \varphi  \right\rangle_{\ll}  \\
            & \qquad +
              \sup_{\substack{\varphi \in \h \\ \|\varphi\|_{\h} = 1}}
            \left\langle \nabla (I - \Deltax)^{-1} \left(D\calE(u_t + r_t) 
                - D\calE(u_t )\right), \nabla \varphi  \right\rangle_{\ll} \\
        &=  \sup_{\substack{\varphi \in \h \\ \|\varphi\|_{\h} = 1}}
            \left\langle (I - \Deltax)^{-1} \left(D\tcalE(u_t + r_t) 
                - D\calE(u_t + r_t)\right), \Delta \varphi  \right\rangle_{\ll}  \\
            & \qquad +
              \sup_{\substack{\varphi \in \h \\ \|\varphi\|_{\h} = 1}}
            \left\langle (I - \Deltax)^{-1} \left(D\calE(u_t + r_t) 
                - D\calE(u_t )\right), \Delta \varphi  \right\rangle_{\ll} \\
        &=  \sup_{\substack{\varphi \in \h \\ \|\varphi\|_{\h} = 1}}
            \left\langle \left(D\tcalE(u_t + r_t) 
                - D\calE(u_t + r_t)\right), (I - \Delta)^{-1}\Delta \varphi  \right\rangle_{\ll}  \\
            & \qquad +
              \sup_{\substack{\varphi \in \h \\ \|\varphi\|_{\h} = 1}}
            \left\langle  \left(D\calE(u_t + r_t) 
                - D\calE(u_t )\right), (I - \Delta)^{-1}\Delta \varphi  \right\rangle_{\ll} \\
        &\leq  
        \sup_{\substack{\varphi \in \h \\ \|\varphi\|_{\h} = 1}}
            \left\langle \left(D\tcalE(u_t + r_t) 
                - D\calE(u_t + r_t)\right), \varphi  \right\rangle_{\ll}  \\
            & \qquad +
              \sup_{\substack{\varphi \in \h \\ \|\varphi\|_{\h} = 1}}
            \left\langle  \left(D\calE(u_t + r_t) 
                - D\calE(u_t )\right),  \varphi  \right\rangle_{\ll} 
        \numberthis \label{eq:upperbound_terms_error_analysis}
    \end{align*}
    Now from Assumption~\ref{assumption:2}, we know that 
    for all $x \in \Omega$ and $u\in \h$ we have the following bounds on the 
    difference of partials of $L$ and $\tL$:
    \begin{equation}
        \label{eq:upperbound_grad_assumption1}
        \sup
            \left\|\nbu \tL(x, u(x), \nabla u(x)) - \nbu L(x, u(x), \nabla u(x) )\right\|_{2} \leq \epsilon_L \|u(x)\|_2,
    \end{equation}
    and 
    \begin{equation}
        \label{eq:upperbound_grad_assumption2}
        \sup
            \left\|\nbnbu \tL(x, u(x), \nabla u(x)) - \nbnbu L(x, u(x), \nabla u(x) )\right\|_{2} \leq \epsilon_L \|u(x)\|_2,
    \end{equation}
    Therefore, note that we can bound the 
    difference of $\pnbu \tL$ and $\pnbu L$ for all $x \in \Omega$
    and $u\in\h$
    as follows,
    \begin{align*}
        &\sup 
        \left\|\pnbu \tL(x, u(x), \nabla u(x)) - \pnbu L(x, u(x), \nabla u(x) )\right\|_{2} \\
        &\leq \sup 
        \left\|\nbnbu \tL(x, u(x), \nabla u(x)) - \nbnbu L(x, u(x), \nabla u(x) )\right\|_{2} + \sup
        \left\|\nbnbu \tL(x, u(x), \nabla u(x)) - \nbnbu L(x, u(x), \nabla u(x) )\right\|_{2} \\
        &\leq 2\epsilon_L\|u(x)\|_2
        \numberthis \label{eq:bound_on_grads}
    \end{align*}
    
    Note that, from Lemma~\ref{lemma:properties_of_EandL}
    we know that for all $u, v$ we can write the inner product $\langle \op(u), v\rangle$ as follows
    \begin{align*}
        \langle \op(u), v\rangle_{\ll} &= 
            \langle \nbnbu L(x, u, \nabla u), v\rangle_{\ll} + 
            \langle \nbu L(x, u, \nabla u), v\rangle_{\ll}\\
        &=
            \langle \pnbu L(x, u, \nabla u), [v, \nabla v]\rangle_{\ll}
        \numberthis \label{eq:example_write}
    \end{align*}
    that is, we we combine $\nbnbu L $ and $\nbu L$ into a single vector 
    $\pnbu L := [\nbu L(x, u, \nabla u), \nbnbu L(x, u, \nabla u)] \in \R^{d+1}$
    and combining $u$ and $\nabla u$ as a vector $[u, \nabla u]$.

    Using upper bound
    in Equation~\ref{eq:bound_on_grads}
    we can upper bound 
    $\sup_{\substack{\varphi \in \h \\ \|\varphi\|_{\h} = 1}}
        \left\langle \left(D\tcalE(u_t + r_t) 
            - D\calE(u_t + r_t)\right), \varphi  \right\rangle_{\ll}$
    (by expanding it as in Equation~\ref{eq:example_write})
    as follows,
    \begin{align*}
    &\sup_{\substack{\varphi \in \h \\ \|\varphi\|_{\h} = 1}}
        \left\langle \left(D\tcalE(u_t + r_t) 
            - D\calE(u_t + r_t)\right), \varphi  \right\rangle_{\ll}\\
    &=\sup_{\substack{\varphi \in \h \\ \|\varphi\|_{\h} = 1}}
        \left\langle 
        \pnbu \tL(x, u_t + r_t, \nabla u_t + \nabla r_t)
        - 
        \pnbu L(x, u_t + r_t, \nabla u_t + \nabla r_t)
        , \left[\varphi, \nabla \varphi \right]  \right\rangle_{\ll}\\
    &=\sup_{\substack{\varphi \in \h \\ \|\varphi\|_{\h} = 1}}
        \left\langle 
        \nbnbu \tL(x, u_t + r_t, \nabla u_t + \nabla r_t)
        - 
        \nbnbu L(x, u_t + r_t, \nabla u_t + \nabla r_t)
        ,  \nabla \varphi   \right\rangle_{\ll}\\
    &=\sup_{\substack{\varphi \in \h \\ \|\varphi\|_{\h} = 1}}
        \left\langle 
        \nbnbu \tL(x, u_t + r_t, \nabla u_t + \nabla r_t)
        - 
        \nbnbu L(x, u_t + r_t, \nabla u_t + \nabla r_t)
        ,   \varphi   \right\rangle_{\ll}\\
    &\qquad +\sup_{\substack{\varphi \in \h \\ \|\varphi\|_{\h} = 1}}
        \left\langle 
        \nbu \tL(x, u_t + r_t, \nabla u_t + \nabla r_t)
        - 
        \nbu L(x, u_t + r_t, \nabla u_t + \nabla r_t)
        ,  \varphi   \right\rangle_{\ll}\\
    &\leq 
    \sup_{\substack{\varphi \in \h \\ \|\varphi\|_{\h} = 1}}
    \epsilon_L \|u_t + r_t\|_{\ll} ( 1 + \pc)\|\varphi\|_{\ll}\\
    &\leq \epsilon_L(1 + \pc) \|u_t + r_t\|_{\ll}\\ 
    &\leq \epsilon_L(1 + \pc)^2 \|u_t + r_t\|_{\h} 
    \numberthis \label{eq:upperbound_firs_Term_error_analysis}
    \end{align*}

    We can similarly bound 
    $\sup_{\substack{\varphi \in \h \\ \|\varphi\|_{\h} = 1}}
        \left\langle \left(D\calE(u_t + r_t) 
            - D\calE(u_t)\right), \varphi  \right\rangle_{\ll}$
    where will use the convexity of the function $L(x, \cdot, \cdot)$
    for all $u\in\h$ to bound the gradient 
    $\pnbu L(x, u_t + r_t, \nabla u_t + \nabla r_t)$ 
    using Taylor's theorem in the following way, 
    \begin{align*}
        \pnbu L(x, u_t + r_t, \nabla u_t + \nabla r_t)
        = \pnbu L(x, u_t, \nabla u_t) 
        + [r_t, \nabla r_t]^T \Dmynabla_{(u, \nabla u)} L(\tilde{x}, u_t(\tilde{x}), \nabla u(\tilde{x}))\\
        \implies
        \pnbu L(x, u_t + r_t, \nabla u_t + \nabla r_t)
        - \pnbu L(x, u_t, \nabla u_t) =
         [r_t, \nabla r_t]^T \Dmynabla_{(u, \nabla u)} L(\tilde{x}, u_t(\tilde{x}), \nabla u(\tilde{x}))
    \end{align*}
    here $\tilde{x} \in \Omega$.
    Therefore, bounding 
    $\sup_{\substack{\varphi \in \h \\ \|\varphi\|_{\h} = 1}}
        \left\langle \left(D\calE(u_t + r_t) 
            - D\calE(u_t)\right), \varphi  \right\rangle_{\ll}$
   we get,
    \begin{align*}
    &\sup_{\substack{\varphi \in \h \\ \|\varphi\|_{\h} = 1}}
        \left\langle \left(D\calE(u_t + r_t) 
            - D\calE(u_t)\right), \varphi  \right\rangle_{\ll}\\
    &=\sup_{\substack{\varphi \in \h \\ \|\varphi\|_{\h} = 1}}
    \left\langle 
        \pnbu L(x, u_t + r_t, \nabla u_t + \nabla r_t)
        - \pnbu L(x, u_t, \nabla u_t), 
        [\varphi, \nabla \varphi]
    \right\rangle_{\ll}\\
    &=\sup_{\substack{\varphi \in \h \\ \|\varphi\|_{\h} = 1}}
    \left\langle 
        [r_t, \nabla r_t]^T\Dmynabla_{(u, \nabla u)}L(\tilde{x}, u(\tilde{x}), \nabla u(\tilde{x})),
        [\varphi, \nabla \varphi]
    \right\rangle_{\ll}\\
    &\leq \sup_{\substack{\varphi \in \h \\ \|\varphi\|_{\h} = 1}}
    \Lambda \left\langle 
        [r_t, \nabla r_t]^T,
        [\varphi, \nabla \varphi]
    \right\rangle_{\ll}\\
    &\leq \sup_{\substack{\varphi \in \h \\ \|\varphi\|_{\h} = 1}}
    \Lambda \left(\|r_t\|_{\ll}\|\varphi\|_{\ll} + \|\nabla r_t\|_{\ll}\|\nabla \varphi\|_{\ll}\right)\\
    &\leq \Lambda \left(1 + \pc\right)^2 \|r\|_{\h}
     \numberthis \label{eq:upperbound_second_term_error_analysis}
    \end{align*}
    Plugging in 
    Equations 
    \eqref{eq:upperbound_firs_Term_error_analysis} 
    and \eqref{eq:upperbound_second_term_error_analysis}
    in \eqref{eq:upperbound_terms_error_analysis}
    we get,
    \begin{align*}
        \left\|(I - \Deltax)^{-1} D\tcalE(u_t + r_t) - D\calE(u_t)\right\|_{\h}
        &\leq  \epsilon_L (1 + \pc)^2\|u_t + r_t\|_{\h}
            + \Lambda(1 + \pc)^2 \|r\|_{\h}\\
        &= (1+ \pc)^2(\epsilon_L + \Lambda)\|r_t\|_{\h} + \epsilon(1 + \pc)^2\|u_t\|
        \numberthis \label{eq:bound_on_h01_error_analysis1}
    \end{align*}
    
    Furthermore, from Lemma~\ref{lemma:proof_for_convergence}
    we have for all $t \in \N$, 
    \begin{align*}
            \calE(u_{t+1})-\calE(u^\star) 
            &\leq \left(1 - \frac{\lambda^6}{(1 + \pc)^8\Lambda^5 }\right)^t\calE(u_0) \\
            &\leq \calE(u_0)
    \end{align*}
    and
    \begin{align*}
        \|u_t - u^\star\|_{\h} 
            &\leq \frac{2}{\lambda}\left(\calE(u_t) - \calE(u_0)\right)\\
            &\leq \frac{2}{\lambda}\calE(u_0)
    \end{align*}
    Hence we have that for all $t \in \N$,
    $$\|u_t\|_{\h} \leq \|u^\star\|_{\h} + \frac{2}{\lambda}\calE(u_0) =: R.$$
    
   Putting this all together, we have
    \begin{equation}
        \label{eq:bound_on_h01_error_analysis}
        \left\|(I - \Deltax)^{-1} D\tcalE(u_t + r_t) - D\calE(u_t)\right\|_{\h}
        \leq (1+ \pc)^2(\epsilon_L + \Lambda)\|r_t\|_{\h} + \epsilon_L(1 + \pc)^2 R
    \end{equation}
    
    Hence using the result from \eqref{eq:bound_on_h01_error_analysis}
    in \eqref{eq:update_equation}
   and unfolding the recursion, we get,
    \begin{align*}
        &\|r_{t+1}\|_{\h} \leq 
            \left(1 + \eta(1 + \pc)^2(\epsilon_L + \Lambda) \right)\|r_t\|_{\h} + (1 + \pc)^2\epsilon_L\eta R \\
        \implies& 
        \|r_{t+1} \|_{\h} 
        \leq \frac{(1 + \pc)^2\epsilon_L\eta R}{\eta (1 + \pc)^2(\epsilon_L + \Lambda)}
        \left(\left(1 +\eta(1 + \pc)^2\left(\epsilon_L + \Lambda)\right)\right)^t - 1\right)\\
        \implies
        &\|r_{t+1}\|_{\h} 
        \leq \frac{\epsilon_L R}{\epsilon_L + \Lambda}
        \left(\left(1 +\eta(1 + \pc)^2\left(\epsilon_L + \Lambda)\right)\right)^t - 1\right)
        \numberthis \label{eq:last_of_error_eq}
    \end{align*}
    as we needed.
\end{proof}

\section{Proofs for Section~\ref{subsec:barron_norm_approximation}: Bounding the Barron Norm}
\subsection{Proof of Lemma~\ref{lemma:barron_norm_recursion}: Barron Norm Increase after One Update}
\label{subsec:proof_of_barron_norm_recursion}
\begin{proof}
    Note that the update equation looks like, 
    \begin{align*}
        \tu_{t+1} 
            &= \tu_t - \eta (I - \Deltax)^{-1} \op(u_t) \\
            &= \tu_t - \eta (I - \Deltax)^{-1} \left(-\nabla \cdot \nbnbu L(x, \tu_t, \nabla \tu_t) + \nbu L(x, \tu_t, \nabla \tu_t) - f \right)\\
            &= \tu_t - \eta (I - \Deltax)^{-1} \left(-\sum_{i=1}^d \partial_i \nbnbu L(x, \tu_t, \nabla \tu_t) + \nbu L(x, \tu_t, \nabla \tu_t) - f\right) 
            \numberthis \label{eq:update_written_in_full}
    \end{align*}
    
    From Lemma~\ref{lemma:barron_norm_algebra}
    we have
    \begin{equation}
        \label{eq:equation_barron_norm_ut}
        \|\nabla \tu_t\|_{\barron}  = \max_{i \in [d]} 
            \|\partial_i \tu_t\|_{\barron} \leq 2\pi W_t\|\tu_t\|_{\barron}
    \end{equation}
    This also implies that 
    \begin{align*}
        \max\{\|\tu_t\|_{\barron}, \|\nabla \tu_t\|_{\barron}\} \leq 2\pi W_t\|\tu_t\|_{\barron}.
    \end{align*}
    Note that since $\tu_t \in \Gamma_{W_t}$ 
    we have $\nabla \tu_t \in \Gamma_{2\pi W_t}$ and 
    $L(x, \tu_t, \nabla \tu_t) \in \Gamma_{2\pi k_{\tL}W_t}$ 
    (from Assumption~\ref{assumption:2}).
    
    Therefore, we can bound the Barron norm as,
    \begin{align*}
        &\left\|(I - \Deltax)^{-1} \left(-\sum_{i=1}^d \partial_i \nbnbu
            L(x, \tu_t, \nabla \tu_t) + \nbu L(x, \tu_t, \nabla \tu_t) - f\right)\right\|_{\barron} \\
        &\stackrel{(i)}{\leq} \left\|-\sum_{i=1}^d \partial_i \nbnbu L(x, \tu_t, \nabla \tu_t)\right\|_{\barron} 
            + \|\nbu L(x, \tu_t, \nabla \tu_t)\|_{\barron} + \|f\|_{\barron}
            \\
        &\stackrel{(ii)}{\leq}d\left\|\partial_i \nbnbu L(x, \tu_t, \nabla \tu_t)\right\|_{\barron}
                + \|\nbu L(x, \tu_t, \nabla \tu_t)\|_{\barron} + \|f\|_{\barron}
        \\
        &\leq 
        d B_{\tL} 2\pi k_{\tL}(2\pi W_t)^{p_{\tL}} \|u\|_{\barron}^{p_{\tL}} 
        + B_{\tL} (2\pi W_t)^{p_{\tL}} \|u\|_{\barron}^{p_{\tL}} 
        + \|f\|_{\barron} \\
        &\leq 
        (2\pi k_{\tL}d +1)B_{\tL} (2\pi W_t)^{p_{\tL}} \|u\|_{\barron}^{p_{\tL}} 
        + \|f\|_{\barron} 
    \end{align*}
    where we use the fact that for a function $h$, we have $\|(I - \Deltax)^{-1}h\|_{\barron} \leq \|h\|_{\barron}$
    from Lemma~\ref{lemma:barron_norm_algebra} in $(i)$
    and the bound from \eqref{eq:equation_barron_norm_ut} in $(ii)$.

    Using the result of \emph{Addition} 
    from Lemma~\ref{lemma:barron_norm_algebra}
    we have
    \begin{align*}
        \|\tu_{t+1}\|_{\barron} &\leq \|\tu_t\|_{\barron} 
        + \eta\left(2\pi k_{\tL}d +1)B_{\tL} (2\pi W_t)^{p_{\tL}} \|u\|_{\barron}^{p_{\tL}} 
        + \|f\|_{\barron}\right) \\
        &\leq \barrononestep 
    \end{align*}
\end{proof}

\subsection{Proof of Lemma~\ref{lemma:final_recursion_lemma}: Final Barron Norm Bound}
\begin{proof}
    From Lemma~\ref{lemma:barron_norm_recursion}
    we have 

    \begin{align*}
        \|\tu_{t+1}\|_{\barron} &\leq \|\tu_t\|_{\barron} 
        + \eta\left((2\pi k_{\tL}d +1)B (2\pi W_t)^{p} \|u\|_{\barron}^{p} 
        + \|f\|_{\barron}\right) \\
        &\leq 
        \left(1 +\eta (2\pi k_{\tL}d +1)B (2\pi W_t)^{p} \right)
        \|u\|_{\barron}^{p} 
        + \eta \|f\|_{\barron} \\
    \end{align*}
    Denoting the constant 
    $A = \left(1 +\eta (2\pi k_{\tL}d +1)B (2\pi W_t)^{p} \right)$
    we have
    \begin{align*}
        \|\tu_{t+1}\|_{\barron} &= A \|\tu_t\|_{\barron}^p + \eta \|f\|_{\barron} \\
        \log\left(\|\tu_{t+1}\|_{\barron}\right) 
            &= \log\left(A \|\tu_t\|_{\barron}^p + \eta \|f\|_{\barron} \right)\\
            &= \log\left(A \|\tu_t\|_{\barron}^p \left(1 + \frac{\eta \|f\|_{\barron}}{A \|\tu_t\|_{\barron}^p}  \right)\right)\\ 
            &\leq \log\left(A \|\tu_t\|_{\barron}^p \left(1 + \frac{\eta \|f\|_{\barron}}{\max\{1, A\|\tu_t\|_{\barron}^p\}}  \right)\right) \\
            &= \log\left(A \|\tu_t\|_{\barron}^p \left(1 + \eta \|f\|_{\barron}  \right)\right)\\
            &= \log\left(\|\tu_t\|_{\barron}^p\right) 
                + \log\left(A\left(1 + \eta \|f\|_{\barron}\right)\right) \\
            &= r \log(\|\tu_t\|_{\barron}) 
                + \log\left(A\left(1 + \eta \|f\|_{\barron}\right)\right) 
        \numberthis \label{eq:final_recursion_proof} 
    \end{align*}
    The above equation is a recursion of the form 
    $$x_{t+1} \leq rx_t + c$$
    which implies
    $$x_{t+1} \leq c \frac{p^t - 1}{p - 1} + p^tx_0.$$
    Therefore the final bound in \eqref{eq:final_recursion_proof}
    is,
    \begin{align*}
        \log\left(\|\tu_{t+1}\|_{\barron}\right) 
            &\leq r \log(\|\tu_t\|_{\barron}) 
                + \log\left(A\left(1 + \eta \|f\|_{\barron}\right)\right) \\
        \implies
        \log\left(\|\tu_{t+1}\|_{\barron}\right) 
            &\leq\frac{r^n - 1}{r - 1} \log\left(A\left(1 + \eta \|f\|_{\barron}\right)\right) 
            + p^t \log(\|\tu_0\|_{\barron}) \\
        \implies \|\tu_{t+1}\|_{\barron} &\leq
            \left(A\left(1 + \eta \|f\|_{\barron}\right)\right)^{\frac{p^t-1}{p-1}}\|\tu_0\|_{\barron}^{p^t}\\
        \implies \|\tu_{t+1}\|_{\barron} &\leq
            \left(
            \left(1 +\eta (2\pi k_{\tL}d +1)B_{\tilde{L}} (2\pi W_t)^{p} \right)
            \left(1 + \eta \|f\|_{\barron}\right)\right)^{\frac{p^t-1}{p-1}}\|\tu_0\|_{\barron}^{p^t}\\
        \stackrel{(i)}{\implies}
        \|\tu_{t+1}\|_{\barron} &\leq
            \left(
            \left(1 +\eta (2\pi k_{\tL}d +1)B_{\tilde{L}} (2\pi k_{\tL}^tW_0)^{p} \right)
            \left(1 + \eta \|f\|_{\barron}\right)\right)^{\frac{p^t-1}{p-1}}\|\tu_0\|_{\barron}^{p^t}\\
        \stackrel{(ii)}{\implies}
        \|\tu_{t+1}\|_{\barron} &\leq
            \left(
            \left(1 +\eta (2\pi k_{\tL}d +1)B_{\tilde{L}} (2\pi k_{\tL}W_0) \right)
            \left(1 + \eta \|f\|_{\barron}\right)\right)^{pt + \frac{p^t-1}{p-1}}\|\tu_0\|_{\barron}^{p^t}\\
        \implies
        \|\tu_{t+1}\|_{\barron} &\leq
            \left(
            \left(1 +\eta 2\pi k_{\tL} W_0 (2\pi kd +1)B_{\tilde{L}}\right)
            \left(1 + \eta \|f\|_{\barron}\right)\right)^{pt + \frac{p^t-1}{p-1}}\left(\max\{1, \|\tu_0\|_{\barron}^{p^t}\}\right)
    \end{align*}
    where we use the fact that $W_t = k_{\tL}^T W_0$ since 
    $\tu_t \in \Gamma_{k_{\tL}^T W_0}$ in step $(i)$
    and 
    use the property that $(1 + x^p) \leq (1 + x)^p$ since $x > 0$
    in step $(ii)$.
\end{proof}

\subsection{Proof of Lemma~\ref{lemma:polynomial_barron_norm_result}}
\label{subsec:proof_barron_norm_polynomial}

\begin{lemma}[
Lemma~\ref{lemma:polynomial_barron_norm_result} restated]
    Let 
    $$
    f(x)
        = \sum_{\alpha, |\alpha|\leq P} \left(A_{\alpha}\prod_{i=1}^d x_i^{\alpha_i}\right)
    $$
    where $\alpha$ is a multi-index and 
    $x \in \R^d$ and 
    $A_\alpha  \in \R$ is a scalar.
    If $g: \R^d \to \R^d$
    is a function such that $g \in \Gamma_W$, 
    then we have $f \circ g \in \Gamma_{PW}$ and 
    the Barron norm can be bounded as,
    $$
        \|f \circ g\|_{\barron} 
            \leq d^{P/2}\left(\sum_{\alpha, |\alpha|=1}^P |A_{\alpha}|^2 \right)^{1/2}\|g\|_{\barron}^P
    $$
\end{lemma}
\begin{proof}
    Recall from Definition~\ref{definition:barron_norm_vector}
    we know
    that for a vector valued function
    $g: \R^d \to \R^d$,
    we have 
    $$ 
    \|g\|_{\barron} = \max_{i \in [d]} \|g_i\|_{\barron}.
    $$

    Then, using Lemma~\ref{lemma:barron_norm_algebra}, we have

    \begin{align*}
        \|f(g)\|_{\barron} 
                &= \left\|\sum_{\alpha, |\alpha|=0}^P A_{\alpha} 
                    \prod_{i=1}^d g_i^{\alpha_i}\right\|_{\barron}\\
                &\leq \sum_{\alpha, |\alpha|=0}^P \left\|A_{\alpha} 
                    \prod_{i=1}^d g_i^{\alpha_i}\right\|_{\barron}\\
                &\leq \sum_{\alpha, |\alpha|=0}^P |A_{\alpha}|\left\| 
                    \prod_{i=1}^d g_i^{\alpha_i}\right\|_{\barron}\\
                &\leq \sum_{\alpha, |\alpha|=0}^P |A_{\alpha}|\left\| 
                    \prod_{i=1}^d g_i^{\alpha_i}\right\|_{\barron}\\
                &\leq \sum_{\alpha, |\alpha|=0}^P |A_{\alpha}| 
                    \left(\prod_{i=1}^d \left\|g_i^{\alpha_i}\right\|_{\barron}\right)\\
                &\leq \sum_{\alpha, |\alpha|=0}^P |A_{\alpha}| 
                    \left(\prod_{i=1}^d \left\|g_i\right\|^{\alpha_i}_{\barron}\right)\\
                &= \sum_{\alpha, |\alpha|=0}^P |A_{\alpha}| 
                    \left(\prod_{i=1}^d \left\|g_i\right\|^{\alpha_i}_{\barron}\right)\\
                &\leq \left(\sum_{\alpha, |\alpha|=0}^P |A_{\alpha}|^2 \right)^{1/2}
                    \left(\sum_{\alpha, |\alpha|=1}^P\left(\prod_{i=1}^d 
                    \left\|g_i\right\|^{\alpha_i}_{\barron}\right)^2\right)^{1/2} 
                \numberthis \label{eq:magnitude_norm_poly_deri_end}
    \end{align*}
    where we have repeatedly used Lemma~\ref{lemma:barron_norm_algebra} and Cauchy-Schwartz in the last line. 
    Using the fact that 
    for a multivariate function $g: \R^d \to \R^d$
    we have for all $i \in [d]$
    $$ \|g\|_{\barron} \geq \|g_i\|_{\barron}.$$
    Therefore, 
    from \eqref{eq:magnitude_norm_poly_deri_end} we get,
    \begin{align*}
        \|f(g)\|_{\barron} 
                &\leq 
                    \left(\sum_{\alpha, |\alpha| = 0}^P |A_{\alpha}|^2 \right)^{1/2}
                    \left(\sum_{\alpha, |\alpha| = 1}^P \left(
                    \left\|g\right\|^{\sum_{i=1}^d \alpha_i}_{\barron}\right)^2\right)^{1/2} \\
                &\leq \left(\sum_{\alpha, |\alpha| = 0}^P |A_{\alpha}|^2 \right)^{1/2}
                    \left(\sum_{\alpha, |\alpha| = 1}^P\left(\|g\|_{\barron}^\alpha\right)^2\right)^{1/2} \\
                &\leq d^{P/2}\left(\sum_{\alpha, |\alpha| = 0}^P |A_{\alpha}|^2 \right)^{1/2}\|g\|_{\barron}^P 
    \end{align*}
    
    Since the maximum power of the polynomial can take is $P$
    from Corollary~\ref{corollary:polynomial_barron}
    we will have $f \circ g \in \Gamma_{PW}$.
\end{proof}
\subsection{Proof of Lemma~\ref{lemma:barron_norm_algebra}: Barron Norm Algebra}
\label{subsec:proof:barron_norm_algebra}
The proof of Lemma~\ref{lemma:barron_norm_algebra} is fairly similar to the proof of Lemma 3.3 in \cite{chen2021representation}---the change stemming from the difference of the Barron norm being considered

\begin{proof}
    We first show the result for \emph{Addition}
    and bound $\|h_1 + h_2\|_{\barron}$,
    \begin{align*}
        \|g_1 + g_2\|_{\barron} 
        &=
                \sum_{\omega \in \Z^d} \left(1 + \|\omega\|_2\right)
                    |\widehat{g_1 + g_2}(\omega)|  \\
        &= 
                \sum_{\omega \in \Z^d} \left(1 + \|\omega\|_2\right)|\hat{g}_1(\omega) + \hat{g}_2(\omega)| \\
        &\leq 
                \sum_{\omega \in \Z^d} \left(1 + \|\omega\|_2\right)|\hat{g}_1(\omega)| 
                +
                 \sum_{\omega \in \Z^d} \left(1 + \|\omega\|_2\right)|\hat{g}_2(\omega)|  \\
        \implies 
        \|h_1 + h_2\|_{\barron}  &\leq \|h_1\|_{\barron} + \|h_2\|_{\barron}.
    \end{align*}
    
    For \emph{Multiplication}, first note that
    multiplication of functions is equal to convolution of the functions in the frequency domain, i.e., 
    for functions $g_1:\R^d \to d$ and $g_2: \R^d \to d$, we have,
    \begin{equation}
        \widehat{g_1 \cdot g_2} = \hat{g}_1 * \hat{g}_2
    \end{equation}
    
    Now, to bound the Barron norm for the multiplication of two functions,
    \begin{align*}
        \|g_1 \cdot g_2\|_{\barron} 
            &= 
                \sum_{\omega \in \Z^d} (1 + \|\omega\|_2)
                    |\widehat{g_1 \cdot g_2}(\omega)| \\
            &= 
                \sum_{\omega \in \Z^d} (1 + \|\omega\|_2)|\hat{g}_1 * \hat{g}_2(\omega)| \\
            &= 
                \sum_{\omega \in \Z^d} \sum_{z \in \Z^d}
                \left(1 + \|\omega\|_2\right)\left|\hat{g}_1(z) \hat{g}_2(\omega - z)\right| \\
            &\leq 
                \sum_{\omega \in \Z^d} \sum_{z \in \Z^d}
                \left(1 + \|\omega - z\|_2 + \|z\|_{2} + \|z\|_2\|\omega -z\|_2\right)\left|\hat{g}_1(z) \hat{g}_2(\omega - z)\right| 
    \end{align*}
        Where we use $\|\omega\|_2 \leq \|\omega -z\|_2 + \|z\|_2$
        and 
        the fact that 
        $$\sum_{\omega} \sum_z \|z\|_2\|\omega - z\|_2 
            |\hat{g}_1(z) \hat{g}_2(\omega -z)| > 0.$$
        
    Collecting the relevant terms together we get,
    \begin{align*}
        \|g_1 \cdot g_2\|_{\barron} 
            &\leq 
                \sum_{\omega \in \Z^d} \sum_{z \in \Z^d}
                \left(1 + \|\omega - z\|_2\right)\cdot
                \left(1 + \|z\|_{2} \right)
                \left|\hat{g}_1(z)\right|\left|\hat{g}_2(\omega - z)\right| \\
            &= 
            \left((1 + \|\omega\|_2)\hat{g}_1(\omega)\right)
                * \left((1 + \|\omega\|_2)\hat{g}_2(\omega)\right)
    \end{align*}
    Hence using Young's convolution
    identity
    from Lemma~\ref{lemma:youngs_identity}
    we have
    \begin{align*}
        \|g_1 \cdot g_2\|_{\barron} 
            &\leq 
            \left(\sum_{\omega \in \R^d} (1 + \|w\|_2) \hat{g}_1(\omega) d\omega\right)
            \left(\sum_{\omega \in \R^d} (1 + \|w\|_2) \hat{g}_2(\omega) d\omega\right)\\
        \implies \|g_1 \cdot g_2\|_{\barron} 
            &\leq \|h_1\|_{\barron}\|h_2\|_{\barron}.
    \end{align*}
    
    In order to show the bound for \emph{Derivative}, since $h \in \Gamma_W$,  
    there exists a function $g: \R^d \to \R$ such that,
    $$g(x) =  \sum_{\|\omega\|_\infty \leq W} e^{2\pi i \omega^Tx } \hat{g}(\omega) d\omega$$
    Now taking derivative on both sides we get,
    \begin{align*}
        \partial_j g(x) &= \sum_{\|\omega\|_\infty \leq W} i e^{i \omega^Tx} 2\pi \omega_j \hat{g}(\omega)
        \numberthis \label{eq:coeff_for_derivative}
    \end{align*}
    This implies that we can upper bound $|\widehat{\partial_ig}(\omega)|$ as
    \begin{align*}
        \widehat{\partial_j g}(\omega) &= i 2\pi \omega_j \hat{g}(\omega) \\
        \implies |\widehat{\partial_j g}(\omega)| &\leq 2\pi W |\hat{g}(\omega)| \numberthis \label{eq:derivative_fourier_transform}
    \end{align*}
    Hence we can bound the Barron norm of $\partial_j h$ as follows:
    \begin{align*}
        \|\partial_j g\|_{\barron} 
        &= 
        \sum_{\|\omega\|_\infty \leq W} 
            \left(1 + \|\omega\|_\infty\right)|\widehat{\partial_j g}(\omega)| d\omega\\
        &\leq 
        \sum_{\|\omega\|_\infty \leq W} (1 + \|\omega\|_\infty) | 2\pi W \hat{g}(\omega)| d\omega\\
        &\leq 
            2\pi W 
        \sum_{\|\omega\|_\infty \leq W} (1 + \|\omega\|_\infty)|\hat{g}(\omega)| d\omega\\
        &\leq 2\pi W \|h\|_{\barron}
    \end{align*}
    
    In order to show the preconditioning, 
    note that for functions $g,f : \Omega^d \to \R$, 
    if 
    $f = (I - \Delta)^{-1} g$
    then we have
    then we have $(I - \Delta)f = g$. 
    Furthermore, by Lemma~\ref{lemma:derivative_of_function}
    we have
    \begin{equation}
        (1 + \|\omega\|_{2}^2) \hat{f}(\omega) 
        = \hat{g}(\omega) \implies \hat{f}(\omega) 
        =\frac{\hat{g}(\omega)}{1 + \|\omega\|_{2}^2}.
    \end{equation}
    
    Bounding $\|(I - \Delta)^{-1}f\|_{\barron}$,
    \begin{align*}
        \|(I - \Delta)^{-1}g\|_{\barron} 
            &= 
            \sum_{\omega \in \Z^d} \frac{1 +\|\omega\|_2}{(1 + \|\omega\|_2^2)} \hat{g}(\omega) d\omega\\
            &\leq 
            \sum_{\omega \in \Z^d} (1 + \|\omega\|_2) \hat{g}(\omega) d\omega\\
        \implies \|(I - \Delta)^{-1}g\|_{\barron}  
        &\leq \|g\|_{\barron}. \qedhere
    \end{align*}
    
\end{proof}

\medskip

\begin{corollary}
    \label{corollary:polynomial_barron}
    Let $g: \R^d \to \R$
    then for any $k \in \N$ 
    we have 
    $\|g^k\|_{\barron} \leq \|g\|_{\barron}^k$.
    Furthermore, if the function $g \in \Gamma_W$
    then the function $g^k \in \Gamma_{kW}$.
\end{corollary}
\begin{proof}
    The result from $\|g^k\|_{\barron}$ follows from 
    the multiplication result in Lemma~\ref{lemma:barron_norm_algebra}
    and we can show this by induction.
    For $n=2$, we have from Lemma~\ref{lemma:barron_norm_algebra} we have,
    \begin{equation}
        \|g^2\|_{\barron} \leq \|g\|_{\barron}^2
    \end{equation}
    Assuming that we have for all $n$ till $k-1$ we have
    \begin{equation}
        \|g^n\|_{\barron} \leq \|g\|_{\barron}^n
    \end{equation}
    for $n=k$ we get, 
    \begin{equation}
        \|g^{k}\|_{\barron} = \|g g^{k-1}\|_{\barron} \leq \|g\|_{\barron} \|g^{k-1}\|_{\barron} \leq \|g\|^{k}_{\barron}.
    \end{equation}

    To show that for any $k$ the function $g^k \in \Gamma_{kW}$, we write $g^k$ in the Fourier basis. We have: 
    
    \begin{align*}
        g^k(x) &= 
            \prod_{j=1}^k \left(\sum_{\|\omega_j\|_\infty \leq W} \hat{g}(\omega_j)  e^{2i \pi \omega_j^T x} d\omega_j \right) \\
            &=\sum_{\|\omega\|_{\infty} \leq kW}
           \left(\sum_{\sum_{l=1}^k \omega_l = \omega} \prod_{j=1}^k \hat{g}(\omega_j) d\omega_1 \dots d\omega_k\right) e^{i 2\pi \omega^T k} d\omega  \label{eq:power1}
    \end{align*}
    In particular, the coefficients with $\|\omega\|_{\infty} > k  W$ vanish, as we needed. 
\end{proof}


\begin{lemma}[Young's convolution identity]
    \label{lemma:youngs_identity}
    For functions $g \in L^p(\R^d)$ and $h \in L^q(\R^d)$
    and 
    $$ \frac{1}{p} + \frac{1}{q} = \frac{1}{r} + 1$$
    where $1 \leq p, q, r \leq \infty$
    we have 
    $$ \|f * g\|_r \leq \|g\|_p \|h\|_q.$$
    Here $*$ denotes the convolution operator.
\end{lemma}

\begin{lemma}
    \label{lemma:derivative_of_function}
    For a differentiable function $f: [0,1]^d \to \R$, such that $f \in L^1(\R^d)$
    we have 
    $$ \widehat{\nabla f}(\omega) = i 2\pi \omega \hat{f}(\omega)$$
\end{lemma}

\section{Existence Uniqueness and Definition of the Solution}
\subsection{Proof of Existence and Uniqueness of Minima}
\label{subsection:proof_of_uniquness}
\begin{proof}
    The proof follows a similar sketch of that provided in ~\cite{fernandez2020regularity} Chapter 3, Theorem 3.3. 

We first show that the minimizer $u^\star$ of the energy functional $\calE(u)$ exists.

    Note that from Definition~\ref{def:energy_functional}
    we have for a fixed $x \in \Omega$
    the function $L(x, \cdot, \cdot)$ is convex and smooth 
    it has a unique minimum, i.e., 
    there exists a $(y_L, z_L) \in \R \times \R^d$
    such that for all $(y, z) \in \R \times \R^d$ 
    we have $L(x, y, z) \geq L(x, y_L, z_L)$ and that 
    $\mynabla L(x, y_L, z_L) = 0$. 
    Furthermore, using \eqref{eq:assumption_hessian_L} from Definition~\ref{def:energy_functional} 
    this also implies the following, 
    $$\lambda \|z - z_L\|_2^2 \leq L(x, y, z) - L(x, y_L, z_L) \leq \Lambda \left(\|y-y_L\|_2^2 + \|z - z_L\|_2^2\right).$$
    Note we can (w.l.o.g) assume that for a fixed $x \in \Omega$ we have,
    $L(x, 0, 0) = 0$, and $\nabla_{y, z} L(x, 0, 0) = 0$ 
    (
    we can redefine $L$ as $\widetilde{L}(x, y, z) = L(x, y+y_L, z+z_L) - L(x, y_L, z_L)$ if necessary), hence the above equation can be simplified 
    to, 
    \begin{equation}
        \label{eq:uniquness_eq1}
        \lambda \|z\|_2^2 \leq L(x, y, z) \leq \Lambda \left(\|y\|_2^2 + \|z\|_2^2\right), \;\; \forall p \in \domain.
    \end{equation}

    Now, we define,
    $$\calE_\circ = \inf \left\{\int_\Omega L(x, v, \nabla v) - fv\; dx: x \in \Omega, v \in \h \right\} $$. Let us first show that $\calE_\circ$ is finite. Indeed, using \eqref{eq:uniquness_eq1} 
    for any $v \in \h$ and $x \in \Omega$, we have 
    \begin{align*}
        \calE(v) &= \int_\Omega L(x, v, \nabla v) - fv \; dx \\
        &\leq \int_\Omega 
        \Lambda \left(\|v(x)\|^2_2 + \|\nabla v(x)\|^2_{2}\right) + \|f(x)v(x)\|_{2} dx \\
        &\leq  \Lambda\left(\|v\|_{\ll}^2 + \|\nabla v\|_{\ll}^2\right) + \|f\|_{\ll} \|v\|_{\ll}
    \end{align*}
    and is thus finite. 

    Moreover, using \eqref{eq:uniquness_eq1} 
    for all $v \in \h$ and $x \in \Omega$, $\calE(v)$ can be lower bounded as
    \begin{align*}
        \calE(v) &= \int_\Omega L(x, v, \nabla v) - fv \; dx \\
        &\geq \int_\Omega \lambda \|\nabla v(x)\|_2 - \|f(x)v(x)\|_2 \; dx \\
        &\geq \lambda \|\nabla v\|_{\ll}^2 - \|f\|_{\ll}\|v\|_{\ll}\\
        &\geq \frac{\lambda}{2} \|\nabla v\|_{\ll}^2 +  \left(\frac{\lambda}{2 \pc} - \frac{1}{C}\right)\|v\|_{\ll}^2 - C \|f\|_{\ll}^2  
            \numberthis \label{eq:existence_pf_lbd}
    \end{align*}
    for some large constant $C$ so that $\lambda / 2 C_p - 1 / C > 0$., where we have used the Poincare inequality (Theorem~\ref{thm:poincare_inequality}) and Cauchy-Schwarz inequality to get the last inequality. 

    Let $\{u_k\}$ where $u_k \in \h$ $\forall k$ 
    define a minimizing sequence of function, that is, 
    we have $\calE(u_k) \to \calE_\circ = \inf_v \calE(v)$ as $k \to 0$. 
    From \ref{eq:existence_pf_lbd}
    we have for all $k$
    $$ \frac{\lambda}{2} \|\nabla u_k\|_{\ll}^2 +  \left(\frac{\lambda}{2 \pc} - \frac{1}{C}\right)\|u_k\|_{\ll}^2
     - C\|f\|_{\ll}^2
        \leq \calE(u_k).$$
    Therefore since $\calE(u_k)$ is bounded, we have that  $\|u_k\|_{\h}$ is uniformly bounded, and thus we can extract a weakly convergent subsequence. With some abuse of notations, let us without loss of generality  assume that $u_k \rightharpoonup u$. 

    We will now show that if $u_k \rightharpoonup u$, 
    $$ \calE(u) \leq \liminf_{k \to \infty}  \calE(u_k) = \calE_\circ$$
    and therefore conclude that the limit $u$ is a minimizer.
    This property is also referred to as weak-lower semi-continuity of $\calE$.

    In order to show the weak-lower semicontinuity of $\calE$ we define the following set,
    $$\mathcal{A}(t) := \{v \in \h: \calE(v) \leq t\}.$$
    Furthermore, note that the functional $\calE(v)$ 
    is convex in $v$ (since the function $L$ is convex and 
    the term $f(x) v(x)$ is linear), and this also implies that the set $\mathcal{A}(t)$
    is convex.

    Further, for any sequence of functions 
    $\{w_k\}$ where $w_k \in \mathcal{A}(t)$
    such that $w_k \to w$
    from Fatou's Lemma,
    \begin{align*}
    \calE(w) = \int_\Omega L(x, w(x), \nabla w(x)) - f(x)w(x) dx 
        \leq \liminf_{k\to\infty} \int_\Omega L(x, w_k(x), \nabla w_k(x)) - f(x)w_k(x) dx \leq t
    \end{align*}
    hence we also have that the function $w \in \mathcal{A}(t)$.
    Therefore the set $\mathcal{A}(t)$ is closed (w.r.t $\h$ norm), and it is convex.
    Since the set $A(t)$ is closed and convex (it is also weakly closed) therefore
    if $w_k \to w$ it also implies that $w_k \rightharpoonup w$ in $\h$.

    Hence, consider a weakly converging sequence in $\h$, i.e., $w_k \rightharpoonup w$
    and define 
    $$t^* := \liminf_{k\to\infty} \calE(w_k)$$
    Now, for any $\varepsilon > 0$, there exists a subsequence $w_{k_{j,\varepsilon}} \rightharpoonup w$
    in $\h$ and $\calE_{w_{k_{j,\varepsilon}}} \leq t^\star + \varepsilon$,
    that is, $w_{k_{j,\varepsilon}} \in \mathcal{A}(t^* + \varepsilon)$. This this is true for all $\epsilon > 0$
    this implies that $\calE(w) \leq t^* = \liminf_{k\to0}\calE$. 
    Hence the function $\calE$ is lower-semi-continuous, and hence the minimizer exists!

    Now to show that the minimum is unique. 
    Note the function $\calE$ is convex in $u$.
    We will prove that the minima is unique by contradiction.
    
    Let $u, v \in \h$ be two (distinct) minima of $\calE$, i.e., we have,
    $\calE(u) = \calE_\circ$ and $\calE(v) = \calE_\circ $.

    Now using the fact that the function $L: \domain \to \R$ is convex, and the minimality of $\calE_\circ$,
    we have for all $x \in \Omega$ we have
    \begin{align*}
        \calE_\circ \leq \calE\left(\frac{u + v}{2}\right) 
        &= \int_\Omega L\left(x, \frac{u(x) + v(x)}{2}, \frac{\nabla u(x) + \nabla v(x)}{2}\right) + f(x)\frac{u(x) + v(x)}{2}\\
        &= \int_\Omega L\left(\frac{x + x}{2}, \frac{u(x) + v(x)}{2}, \frac{\nabla u(x) + \nabla v(x)}{2}\right) + f(x)\frac{u(x) + v(x)}{2}\\
        &\leq \int_\Omega \frac{1}{2}\left(L\left(x, u(x), \nabla u(x)\right) + u(x)\right) 
            + \int_\Omega \frac{1}{2}\left(L\left(x, v(x), \nabla v(x)\right) + v(x)\right) \\
        &\leq \frac{1}{2}\calE(u) + \frac{1}{2}\calE(v)\\
        \implies& \calE_\circ \leq \calE\left(\frac{u + v}{2}\right)  \leq \frac{1}{2}\calE(u) + \frac{1}{2}\calE(v) = \calE_\circ.
    \end{align*}
    The last inequality is a contradiction and therefore the minima is unique.
\end{proof}
\subsection{Proof of Lemma~\ref{lemma:derivation_of_nonlinear_variational_PDE}: 
Nonlinear Elliptic Variational PDEs}
\label{subsec:proof:derivation_of_nonlinear_variational_PDE}

\begin{proof}[Proof of Lemma~\ref{lemma:derivation_of_nonlinear_variational_PDE}]
    If the function $u^\star$ minimizes the energy functional in Definition \ref{def:energy_functional}
    then we have for all $\epsilon \in \R$
    $$ \calE(u) \leq \calE(u + \epsilon \varphi)$$
    where $\varphi \in C^\infty_c(\Omega)$.
    That is, we have a minima at $\epsilon = 0$ and taking a derivative w.r.t $\epsilon$ and using Taylor expansion
    we get,
    \begin{align*}
        d\calE[u](\varphi)
        &= \lim_{\epsilon \to 0} \frac{\calE(u + \epsilon \varphi) - \calE(u)}{\epsilon} = 0\\
        &= \lim_{\epsilon \to 0} \frac{\int_\Omega L(x, u + \epsilon \varphi, \nabla u + \epsilon \nabla \varphi) - f(x)\left(u(x) + \epsilon \varphi(x)\right) - L(x, u, \nabla u) + f(x)u(x) \; dx}{\epsilon}\\
        &= \lim_{\epsilon \to 0} \frac{\int_\Omega 
        L(x, u + \epsilon \varphi, \nabla u) + \nbnbu L(x, u + \epsilon \varphi, \nabla u) + r_1(x) - \epsilon f(x)\epsilon \varphi(x) - L(x, u, \nabla u)  \; dx}{\epsilon}\\
        &= \lim_{\epsilon \to 0} \frac{\int_\Omega 
        L(x, u, \nabla u) + \epsilon\nbu L(x, u, \nabla u)\varphi + r_2(x)}{\epsilon}\\
        &\qquad + \lim_{\epsilon \to 0}\frac{\epsilon \nbnbu L(x, u , \nabla u)\nabla \varphi + \epsilon^2 \nbu\nbnbu L(x, u, \nabla u)\nabla \varphi \cdot \varphi + r_1(x) - \epsilon f(x)\epsilon \varphi(x) - L(x, u, \nabla u)  \; dx}{\epsilon}\\
        &= \lim_{\epsilon \to 0} \frac{\int_\Omega 
            \epsilon \nbnbu L(x, u, \nabla u)  \nabla \varphi 
            +\epsilon \nbu L(x, u, \nabla u) u 
            + r_1(x) + r_2(x) - \epsilon f(x)\varphi(x) \; dx }{\epsilon}
        \numberthis \label{eq:direction_der_pf_eq1}
        \end{align*}
    where for all $x \in \Omega$ we have,
    \begin{align*}
    |r_1(x)| &\leq 
        \frac{\epsilon^2}{2}\sup_{y \in \Omega} 
        \left|\left(\left(\nabla u(x)\right)^T\Dnbnbu L(y, u + \epsilon \varphi, \nabla u)\nabla u(x)\right)\right|\\
        &\leq \frac{\Lambda\epsilon^2}{2} \|\nabla u(x)\|_2^2
        \numberthis \label{eq:direction_der_pf_eq2}
    \end{align*}
    Similarity we have,
    \begin{align*}
    |r_2(x)| &\leq 
        \frac{\epsilon^2}{2}\sup_{y \in \Omega} 
        \left|\nbu L(y, u , \nabla u) u(x)^2\right|\\
        &\leq \frac{\Lambda\epsilon^2}{2}  u(x)^2
        \numberthis \label{eq:direction_der_pf_eq2_2}
    \end{align*}
    Using results from \eqref{eq:direction_der_pf_eq1} and \eqref{eq:direction_der_pf_eq2_2} in  Equation \eqref{eq:direction_der_pf_eq2} 
    and taking $\epsilon \to 0$,
    the derivative in the direction of $\varphi$
    is,
    \begin{align*}
        d\calE[u](\varphi) 
        &= \lim_{\epsilon \to 0} \frac{\int_\Omega  \nbnbu L(x, u, \nabla u) \nabla \varphi 
            +\nbu L(x, u, \nabla u)u - f(x)\varphi(x)\; dx}{\epsilon}
    \end{align*}
    Since $\epsilon \to 0$ the final derivative is of the form,
    \begin{equation}
        \label{eq:operator_in_normal form}
        d\calE[u](\varphi)
            =\int_\Omega \bigg(\nbnbu L(x, u, \nabla u) \nabla \varphi  
                + \nbu L(x, u, \nabla v)\varphi - f \varphi\bigg) dx = 0.
    \end{equation}
    We
    will now use the following integration by parts identity,
    for functions $r:\Omega \to \R$ such that 
    and $s: \Omega \to \R$, and $r,s \in \h$,
    \begin{equation}
        \label{eq:green_identity_final}
         \int_\Omega \frac{\partial r}{\partial x_i} s dx 
         = - \int_\Omega r \frac{\partial s}{\partial x_i} dx 
            + \int_{\partial \Omega}rs n d\Gamma
    \end{equation}
    where $n_i$ 
    is a normal at the boundary and 
    $d \Gamma$ is an infinitesimal element
    of the boundary $\partial \Omega$.
    
    Using the identity in \eqref{eq:green_identity_final} in \eqref{eq:operator_in_normal form}
    we get,
    \begin{align*}
        d\calE[u](\varphi) 
            &=\int_\Omega \bigg(\nbnbu L(x, u, \nabla u) \nabla \varphi + 
                + \nbu L(x, u, \nabla v)\varphi - f \varphi\bigg) dx \\
            &=\int_\Omega \bigg(\sum_{i=1}^d \left(\nbnbu L(x, u, \nabla u) \right)_i\partial_i \varphi
                + \nbu L(x, u, \nabla v)\varphi - f \varphi\bigg) dx \\
            &=\int_\Omega \bigg(\sum_{i=1}^d - \partial_i \left(\nbnbu L(x, u, \nabla u) \right)_i \varphi
                + \nbu L(x, u, \nabla v)\varphi - f \varphi\bigg) dx \\
            &= \int_\Omega \bigg(-\nablax \cdot \left(\nbnbu L(x, u, \nabla u) \right)\varphi + \nbu L(x, u, \nabla u)\varphi - f\varphi\bigg) dx = 0\\
        \implies d\calE[u](\varphi) 
            &= \int_\Omega \bigg(-\divergence\left(\nbnbu L(x, u, \nabla u) \right)\varphi + \nbu L(x, u, \nabla u)\varphi - f\varphi\bigg) dx = 0\\
    \end{align*}

    That is the minima for the energy functional is reached at a $u$
    which solves the following PDE,
    $$d\calE(u) := -\divergence\left(\nbnbu L(x, u, \nabla u)\right) + \nbu L(x, u, \nabla u)= f.$$
    where we define $d\calE(\cdot)$ as the operator 
    $-\divergence\left(\nbnbu L(x, \cdot, \nabla \cdot)\right) + \nbu L(x, \cdot, \nabla \cdot) $.

\end{proof}

\subsection{Proof of Lemma~\ref{lemma:poincare}: Poincare constant of Unit Hypercube}
\label{subsec:poincare_proof}
\begin{proof}[Proof of Lemma \ref{lemma:poincare}]
    We use the fact that the Poincare constant is the smallest 
    eigenvalue of $\Delta$, i.e., 
    $$ \frac{1}{\pc}
    := \inf_{u \in \ll} \frac{\|\Delta u\|_{\ll}}{\|u\|_{\ll}}
    .
    $$
    
    Note that the
    eigenfunctions of $\Delta$ for the domain $\Omega:= [0,1]^d$
    are defined as 
    $$ \phi_\omega(x) = \prod_{i=1}^d \sin(\pi i \omega_i x_i), \quad \forall \omega \in \Z^d \;\;\&\;\; x \in \Omega.$$
    
    Furthermore, this also implies that for all $\omega \in \Z^d$ we have,
    $$ \Delta \phi_\omega = \pi^2 \|\omega\|_2^2 \phi_\omega.$$
    We can expand any function $u \in \h$ 
    in terms of $\phi_\omega$ as $u(x) = \sum_{\omega \in \Z^d} d_{\omega} \phi_\omega(x)$ where $d_\omega = \langle u, \phi_\omega\rangle_{\ll}$.
    
    Note that for all $x \in \Omega$, we have,
    \begin{align*}
    \Delta u(x)
    &= \sum_{\omega \in \Z^d} \pi^2 \|\omega\|_2^2 d_\omega \phi_\omega(x).
    \end{align*}
    Taking square $\ll$ norm on both sides,
    we get,
    \begin{align*}
        \|\Delta u\|_{\ll}^2 
        &= \pi^4 \left\| \sum_{\omega \in \Z^d} \|\omega\|_2^2 d_\omega \phi_\omega \right\|_{\ll}^2\\
        &\stackrel{(i)}{\geq} \pi^4d^2 
            \left\|\sum_{\omega \in \Z^d} d_\omega \phi_\omega \right\|_{\ll}^2\\
        &\stackrel{(ii)}{=} \pi^4d^2 \|u\|_{\ll}^2 \\
        \implies
        \frac{\|\Delta u\|_{\ll}}{\|u\|_{\ll}} &\geq \pi^2d 
    \end{align*}
    where we use the fact that $\|\omega\|_2 \geq \sqrt{d}$
    (since $\forall i \in [d]$ we have $\omega_i \in \N$)
    in step $(i)$,
    and use the orthogonality of $\{\phi_\omega\}_{\omega \in \Z^d}$
    in $(ii)$. Moreover, it's easy to see that equality can be achieved by taking $u = \phi_{(1,1,\dots, 1)}$. 
    
    Hence the Poincare constant can be calculated as,
    \begin{align*}
    \frac{1}{\pc}
    := \inf_{u \in \ll} \frac{\|\Delta u\|_{\ll}}{\|u\|_{\ll}}
    = \pi^2 d\\
    \implies \pc = \frac{1}{\pi^2 d}.
    \end{align*}
\end{proof}
\section{Important Helper Lemmas}
\begin{lemma}
    \label{lemma:dual_norm}
    The dual norm of $\|\cdot\|_{\h}$ is $\|\cdot\|_{\h}$.
\end{lemma}
\begin{proof}
    If $\|u\|_*$ denotes the dual norm of $\|u\|_{\h}$, by definition we have,
    \begin{align*}
        \|u\|_* &= \sup_{\substack{v \in \h \\ \|v\|_{\h} =1}}
            \langle u, v\rangle_{\h}\\
             &= \sup_{\substack{v \in \h \\ \|v\|_{\h} =1}}
            \langle \nabla u, \nabla v\rangle_{\ll}\\
             &\leq \sup_{\substack{v \in \h \\ \|v\|_{\h} =1}}
             \|\nabla u\|_{\ll} \|\nabla v\|_{\ll}\\
             &= \|\nabla u\|_{\ll}
    \end{align*}
    where the inequality follows by Cauchy- Schwarz. On the other hand, equality can be achieved by taking $v = \frac{u}{\|\nabla u\|_2}$. Thus, $\|u\|_* = \|\nabla u\|_{\ll} = \|u\|_{\h}$ as we wanted.
\end{proof}
\subsection{Useful properties of Laplacian and Laplacian Inverse}

\begin{lemma}
    \label{lemma:self-adjoint}
    The operator $(-\Delta)^{-1}$
    is self-adjoint. 
\end{lemma}
\begin{proof}
    Note that since the operator $(-\Delta)^{-1}$ is bounded, 
    to show that it is self-adjoint, we only need to show that 
    the operator is also symmetric, i.e., 
    for all $u, v \in \h$
    we have
    $$\langle (-\Delta)^{-1} u, v \rangle_{\ll} = \langle u, (-\Delta)^{-1} v\rangle_{\ll}.$$

    To show this, we first show that the operator $\Delta$ is symmetric. 
    i.e, we have 
    \begin{equation}
        \label{eq:self-adjoint_proof}
        \langle -\Delta u, v \rangle_{\ll} = \langle u, -\Delta v\rangle_{\ll} 
    \end{equation}
    This is a direct consequence of the Green's Identity where 
    for functions $u, v \in C^\infty_0$ the following holds,
    \begin{align*}
        \int_\Omega - (\Delta u) v dx &=  \int_\Omega \nabla u \cdot \nabla v dx + \int_{\partial \Omega} \frac{\partial u}{\partial n} v d \Gamma \\
        &=  \int_\Omega \nabla u \cdot \nabla v dx \\
        &=  - \int_\Omega u  \Delta v dx  + \int_{\partial \Omega} \frac{\partial v}{\partial n} u d\Gamma 
    \end{align*}
    where we use the fact that since
    $u, v \in \h$ we have
    $u(x) = 0$ and $v(x) = 0$ for all $x \in \partial \Omega$.

    Now, 
    taking $\tilde{u} = -\Delta u$ and $\tilde{v} = (-\Delta)^{-1} v$
    from Equation \eqref{eq:self-adjoint_proof}
    we get,
    \begin{align*}
        \langle -\Delta u, v \rangle_{\ll} &= \langle u, \Delta v\rangle_{\ll} \\
        \langle \tilde{u}, (-\Delta)^{-1} \tilde{v}\rangle_{\ll} 
            &= \langle (-\Delta)^{-1} \tilde{u}, \tilde{v}\rangle_{\ll}.
    \end{align*}
    Hence we have that the operator $(-\Delta)^{-1}$ is symmetric and bounded and therefore is self-adjoint.
\end{proof}

\begin{lemma}
    \label{lemma:div_equiv}
    Given a vector valued function $f: \R^d \to \R^d$, 
    such that $f \in C^2$
    the following identity holds,
    \begin{equation}
        \label{eq:div_equiv}
        \nabla \divergence (f) =  \divergence( \nabla f).
    \end{equation}
\end{lemma}
\begin{proof}
    We first simplify the right hand side of Equation \eqref{eq:div_equiv}.
    Note that since $\nabla f: \R^d \to \R^{d\times d}$ is a is a matrix valued function 
    the divergence of $\nabla f$ is going to be vector valued. 
    More precisely for all $x\in \Omega$,
    $-\divergence(\nabla f)$
    is defined as
    \begin{align*}
        \divergence(\nabla f(x)) &= 
            \left[\sum_{j=1}^d \partial_j [\nabla f(x)]_i \right]_{i=1}^d\\
            &= 
            \left[\sum_{j=1}^d \partial_j \partial_i f(x) \right]_{i=1}^d
        \numberthis \label{eq:div_equiv_term1}
    \end{align*}
    where for a vector valued function the notation $[g(x)]_i$
    denotes its $i^{\text{th}}$ coordinate, and the notation
    $[g(x)]_{i=1}^d := \left(g(x)_1, g(x)_2, \cdots, g(x)_d\right)$
    denotes a $d$ dimensional vector.

    Now, simplifying the left hand side, for all $x \in \Omega$
    we get,
    \begin{align*}
        \nabla \divergence(f(x)) &= \nabla \left(\sum_{j=1}^d \partial_j f(x)\right)\\
        &= \left[\partial_i \left(\sum_{j=1}^d \partial_j f(x)\right)\right]_{i=1}^d\\
        &= \left[\left(\sum_{j=1}^d \partial_i \partial_j f(x)\right)\right]_{i=1}^d
        \numberthis \label{eq:div_equiv_term2}
    \end{align*}
    Since the term in \eqref{eq:div_equiv_term1} is equal to \eqref{eq:div_equiv_term2}
    we have 
    $\nabla \divergence (f) =  \divergence( \nabla f)$.
\end{proof}

\begin{lemma}
    \label{lemma:nabla_div}
    For a function $g: \R^d \to \R$ such that $g \in C^3$
    the following identity holds,
    $$\Delta \nabla g = \nabla \Delta g$$
\end{lemma}
\begin{proof}
    The term $\Delta \nabla g$ can be simplified as follows,
    \begin{align*}
        \Delta \nabla g 
            &= \Delta \left(\frac{\partial f}{\partial x_1}, \frac{\partial f}{\partial x_2}, \cdots, \frac{\partial f}{\partial x_d}\right)\\
            &= \Delta \left[\frac{\partial f}{\partial x_i}\right]_{i=1}^d\\
            &= \left[\Delta \frac{\partial f}{\partial x_i}\right]_{i=1}^d\\
            &= \left[\sum_{j=1}^d \frac{\partial}{\partial x_j^2}\frac{\partial f}{\partial x_i}\right]_{i=1}^d\\
            &= \left[\sum_{j=1}^d \frac{\partial^2 f}{\partial x_j^2\partial x_i}\right]_{i=1}^d
            \numberthis \label{eq:nabla_div_1}
    \end{align*}
    Further, $\nabla \Delta g$ can be simplified as follows,
    \begin{align*}
        \nabla \Delta g
            &= \nabla \left(\sum_{j=1}^d \frac{\partial g}{\partial x_j^2}\right)\\
            &= \left[
                \sum_{j=1}^d \frac{\partial }{\partial x_1}\frac{\partial g}{\partial x_j^2},
                \sum_{j=1}^d \frac{\partial }{\partial x_2}\frac{\partial g}{\partial x_j^2},
                \cdots,
                \sum_{j=1}^d \frac{\partial }{\partial x_d}\frac{\partial g}{\partial x_j^2},
                \right]\\
            &= \left[
                \sum_{j=1}^d \frac{\partial^2 g}{\partial x_i \partial x_j^2},
                \right]_{i=1}^d
            \numberthis \label{eq:nabla_div_2}
    \end{align*}
    Since \eqref{eq:nabla_div_1} is equal to \eqref{eq:nabla_div_2} it implies that 
    $$\Delta \nabla g = \nabla \Delta g.$$
\end{proof}

\begin{corollary}
    \label{corollary:nabla_grad_equiv}
    For all vector valued function $f: \R^d \to \R^d$
    functions the following holds,
    \begin{equation}
        \label{eq:nabla_grad_equiv_eq}
        \nabla (-\Delta)^{-1} \divergence(f) = (-\Delta)^{-1} \divergence(\nabla f).
    \end{equation}
\end{corollary}
\begin{proof}
    We know from Lemma~\ref{lemma:div_equiv} that for a vector valued function $f: \R^d \to \R^d$
    that we have 
    $$\nabla \divergence(f) = \divergence(\nabla f).$$
    Now, using for a fact that any function $g$ can be written as,
    $g = (-\Delta)(-\Delta)^{-1}g$ we get,
    \begin{align*}
        &\nabla \divergence(f) = \divergence(\nabla f)\\
        \implies&
        \nabla (-\Delta)(-\Delta)^{-1} \divergence(f) = \divergence(\nabla f)\\
        \stackrel{(i)}{\implies}&
        (-\Delta)\nabla (-\Delta)^{-1} \divergence(f) = \divergence(\nabla f)\\
        \implies &
        \nabla (-\Delta)^{-1}\divergence(f) = (-\Delta)^{-1}\divergence(\nabla f)
    \end{align*}
    where $(i)$ follows from Lemma~\ref{lemma:nabla_div}, i.e., for any function $g\in C^3$,
    we have,
    $\nabla \Delta g = \Delta \nabla g$.
\end{proof}

\subsection{Some properties of Sub-Matrices}
\begin{lemma}
    \label{lemma:submatrix_lemma}
    Given matrices $A \in \R^{d\times d}$
    and $B \in \R^{d \times d}$
    if we have $A \preceq B$
    then for any set of indices $U \subseteq \{1, 2, \cdots d\}$
    where $|U| = n \leq d$
    then
    for all $y\in \R^n$
    we have 
    $y^T A_{U} y \leq y^TB_U y$.
    where $A_U = A_{i,j}$ for all $i,j \in U$.
    Similarly if 
    if we have $A \succeq B$
    for all $y\in \R^n$
    we have,
    $y^T A_{U} y \geq y^TB_U y$.
\end{lemma}
\begin{proof}
    We will show that 
    $A\preceq B \implies A_U \preceq B_U$. 
    The proof for 
    $A\succeq B \implies A_U \succeq B_U$ will follow similarly.
    
    Without loss of generality we can assume that $U = \{1, 2, \cdots n\}$
    and a set $V=\{n, \cdots d\}$,
    where $n \leq d$.
    Since $A \preceq B$ we know that there exists 
    $x \in \R^d$
    we have $x^TAx \leq x^TBx$.

    For all $y \in \R^d$ define 
    $x:=(y,\bm{0}_{d-n})$,
    and let 
    $A_{U, V} = A_{i,j}$ be $i\in U$ and $j \in V$
    \begin{align*}
        &\begin{bmatrix}
            y & \bm{0} \\
        \end{bmatrix}^T
        \begin{bmatrix}
            A_{U} & A_{U,V} \\
            A_{V, U} & A_{V}
        \end{bmatrix}
        \begin{bmatrix}
            y & \bm{0} \\
        \end{bmatrix}^T
        \leq
        \begin{bmatrix}
            y & \bm{0} \\
        \end{bmatrix}^T
        \begin{bmatrix}
            B_{U} & B_{U,V} \\
            B_{V, U} & B_{V}
        \end{bmatrix}
        \begin{bmatrix}
            y & \bm{0} \\
        \end{bmatrix}^T
        \\
        \implies&
        \begin{bmatrix}
            y & \bm{0} \\
        \end{bmatrix}^T
        \begin{bmatrix}
            B_{U} - A_{U} & B_{U,V} - A_{U,V}\\
            B_{V, U} - A_{V, U} & B_V - A_{V}
        \end{bmatrix}
        \begin{bmatrix}
            y & \bm{0} \\
        \end{bmatrix}^T
        \geq 0\\
        \implies &
        y^T(B_U - A_U)y \geq 0
    \end{align*}
    Since we have for all $y \in \R^n$ we have
    $y^T(B_U - A_U)y \geq 0$, 
    therefore this implies that $A_U \preceq B_U$.
\end{proof}


\end{document}